\documentclass[pdflatex,sn-mathphys-num]{sn-jnl}% Math and Physical Sciences Numbered Reference Style
\usepackage{graphicx}%
\usepackage{multirow}%
\usepackage{amsmath,amssymb,amsfonts}%
\usepackage{amsthm}%
\usepackage[title]{appendix}%
\usepackage{xcolor}%
\usepackage{textcomp}%
\usepackage{manyfoot}%
\usepackage{booktabs}%
\usepackage{url}
\usepackage{hyperref}
\usepackage{listings}%
\theoremstyle{thmstyleone}%
\newtheorem{theorem}{Theorem}%  meant for continuous numbers
\theoremstyle{thmstyletwo}%

\theoremstyle{thmstylethree}%

\begin{document}

\title[Latent dynamics for PDE trajectory reconstruction]{Latent dynamics for reconstructing fine-scale PDE trajectories from coarse observations}

%%=============================================================%%
%% GivenName	-> \fnm{Joergen W.}
%% Particle	-> \spfx{van der} -> surname prefix
%% FamilyName	-> \sur{Ploeg}
%% Suffix	-> \sfx{IV}
%% \author*[1,2]{\fnm{Joergen W.} \spfx{van der} \sur{Ploeg}
%%  \sfx{IV}}\email{iauthor@gmail.com}
%%=============================================================%%
\author[1]{\fnm{Chenhao} \sur{Si}}
\equalcont{These authors contributed equally to this work.}
\email{chenhao.si@link.cuhk.edu.cn}

\author[2,3]{\fnm{Zishan} \sur{Qian}}
\email{qianzishan@bimsa.cn}
\equalcont{These authors contributed equally to this work.}

\author[3,4]{\fnm{Shing-Tung} \sur{Yau}}

\author*[3,4]{\fnm{Rongling} \sur{Wu}}
\email{ronglingwu@bimsa.cn}

\author*[1]{\fnm{Ming} \sur{Yan}}
\email{yanming@cuhk.edu.cn}

\affil*[1]{\orgdiv{School of Data Science},
\orgname{The Chinese University of Hong Kong, Shenzhen},
\orgaddress{\city{Shenzhen}, \country{China}}}

\affil[2]{\orgname{Institute of Statistics and Big Data, Renmin University},
\orgaddress{\city{Beijing}, \country{China}}}

\affil[3]{\orgname{Beijing Key Laboratory of Topological Statistics and Applications for Complex Systems, Beijing Institute of Mathematical Sciences and Applications},
\orgaddress{\city{Beijing}, \country{China}}}

\affil[4]{\orgname{Yau Mathematical Sciences Center, Tsinghua University},
\orgaddress{\city{Beijing}, \country{China}}}

%%==================================%%
%% Sample for unstructured abstract %%
%%==================================%%

\abstract{
Coarse PDE observations omit spatial detail and intermediate states of evolving fields. We present Structured Latent Dynamics (SLD), a framework that reconstructs high-resolution trajectories from spatially and temporally downsampled simulations. SLD interpolates the input in time, adds a learned temporal correction and uses the resulting sequence to guide a hidden state on the coarse grid. Each update of the latent state combines a learned gated response with horizontal and vertical central differences and a five-point Laplacian of the state. The three spatial terms have learned channel-wise coefficients. Two Haar wavelet synthesis branches map the hidden state and refined coarse input to the high-resolution field. We evaluate SLD on two-dimensional Gray--Scott, complex Ginzburg--Landau and viscous Burgers systems, reconstructing 100 output frames from 25 input frames with eightfold refinement in each spatial direction. SLD achieves lower mean relative $L_2$ error and mean absolute error than unstructured latent dynamics and ConvLSTM baselines with comparable parameter counts, both on held-out trajectories and on later windows of trajectories used during training. SLD also reduces errors in regions with large reference gradients and across most or all of the evaluated later time intervals. A representative Gray--Scott comparison shows lower spectral error over most intermediate and high wavenumbers. Removing the fixed spatial response increased reconstruction error, supporting its contribution to the SLD architecture.}

\keywords{PDE trajectory reconstruction, latent dynamics, spatiotemporal reconstruction, scientific machine learning}

%%\pacs[JEL Classification]{D8, H51}

%%\pacs[MSC Classification]{35A01, 65L10, 65L12, 65L20, 65L70}

\maketitle

\section{Introduction}
\label{sec:Introduction}

Partial differential equations (PDEs) describe transport, reaction, diffusion and pattern formation across science and engineering. Their numerical approximation entails a compromise between resolution and cost: fine grids and short time steps are needed to represent steep gradients, interfaces, shocks, vortices and multiscale interactions, but they rapidly increase memory and computational requirements \cite{LeVeque2002FiniteVolume,Canuto2006SpectralMethods,Davidson2015Turbulence,Warner2010NumericalWeather}. This compromise is especially consequential when downstream quantities depend on derivatives or spectral content, because a coarse field may retain the large-scale pattern while omitting relevant small-scale structure \cite{Pope2000TurbulentFlows,Sreenivasan1999FluidTurbulence}. Reconstructing fine-scale PDE trajectories from coarse observations therefore requires recovery of both spatial detail and its evolution between observed times.

Machine learning for PDEs addresses several objectives. Sparse regression and symbolic approaches seek governing equations or reduced coordinates from data \cite{Brunton2016SINDy,Rudy2017PDEFIND,Cranmer2020DiscoveringSymbolicModels,Champion2019Coordinates}. Physics-informed networks incorporate differential equations into the optimization objective, while related architectures encode discrete operators or learn discretizations from data \cite{Raissi2019PINNs,Karniadakis2021PINNs,long2018pde,long2019pde,BarSinai2019DataDrivenDiscretizations,Rao2023PeRCNN}. Learned closures and solver components can also accelerate simulation or correct under-resolved dynamics \cite{Duraisamy2019TurbulenceML,Kochkov2021MLAcceleratedCFD,Thuerey2021PhysicsBasedDL}. The reconstruction problem considered here uses observations of an evolving coarse field to infer the corresponding high-resolution trajectory.

Operator learning and learned simulators approximate families of PDE solution maps. DeepONet and Fourier neural operators provide constructions for mappings between function spaces \cite{Lu2021DeepONet,li2021fourier,Kovachki2023NeuralOperator}; subsequent work has introduced adaptive spectral mixing, wavelet representations, geometric deformations and attention-based operators \cite{Guibas2021AdaptiveFNO,Gupta2021Multiwavelet,Li2023GeoFNO,Cao2021GalerkinTransformer,Hao2023GNOT,Tran2023FFNO}. Graph-based simulators and message-passing PDE solvers accommodate meshes and local interactions \cite{Pfaff2021MeshGraphNets,Brandstetter2022MessagePassingPDE}. Community benchmarks have broadened the range of equations, geometries and parameter regimes used to evaluate these models \cite{Takamoto2022PDEBench,Ren2025SuperBench,Ohana2024TheWell}.

Temporal reconstruction requires information to be carried between output times while remaining consistent with the observed coarse evolution. Recurrent networks, reservoir models and latent differential equations propagate high-dimensional states or learn compact dynamical coordinates \cite{Hochreiter1997LSTM,Shi2015ConvLSTM,Vlachas2018DataDrivenForecasting,Pathak2018ModelFreePrediction,Chen2018NeuralODE,Rubanova2019LatentODE,Kidger2020NeuralCDE}. Koopman-inspired and sparse latent models further show how nonlinear evolution can be represented in learned coordinates \cite{Lusch2018Koopman,Champion2019Coordinates}. We use a latent state to carry information between output times and update it using features extracted from the coarse input sequence.

Scientific reconstruction and downscaling studies have demonstrated that learned models can recover spatial detail in turbulence, reacting flows and climate fields \cite{Fukami2019SuperResolutionFlow,Deng2019CNNFluidSR,Bode2021UsingDLTurbulence,Stengel2020ClimateSR}. More recent work has coupled spatial refinement with temporal interpolation, physical residuals or learned latent evolution \cite{Ren2023PhySR,kontolati2024learning,Rao2023PeRCNN}. In PDE reconstruction, errors in interface positions, spectral power and regions with large gradients describe aspects of the solution that global field errors can obscure. We test whether fixed directional differences and a Laplacian improve the update of a latent state guided by coarse observations. These operators act on learned features, and their channel-wise weights are learned during training. 

Here we introduce \textbf{Structured Latent Dynamics} (SLD) for reconstructing high-resolution PDE trajectories from coarse observations (Fig.~\ref{fig:arch}). Its state update combines a learned gated response with fixed directional differences and a Laplacian of the latent state. Learned coefficients set the contribution of each spatial operator in each channel. Temporal refinement supplies guidance for each update, and two Haar wavelet synthesis branches reconstruct the output from the updated state and refined coarse input. We evaluate SLD on two-dimensional Gray--Scott, complex Ginzburg--Landau and viscous Burgers systems, mapping 25 coarse frames to 100 frames with eightfold spatial refinement. SLD achieves lower mean reconstruction errors than Unstructured Controlled Latent Dynamics (UCLD) and ConvLSTM with comparable parameter counts in both evaluation protocols. On Gray--Scott, removing the structured spatial term approximately doubles the mean relative $L_2$ error. These results support the proposed update rule within the SLD reconstruction framework.

\begin{figure}[!htb]
\centering
\includegraphics[width=\linewidth]{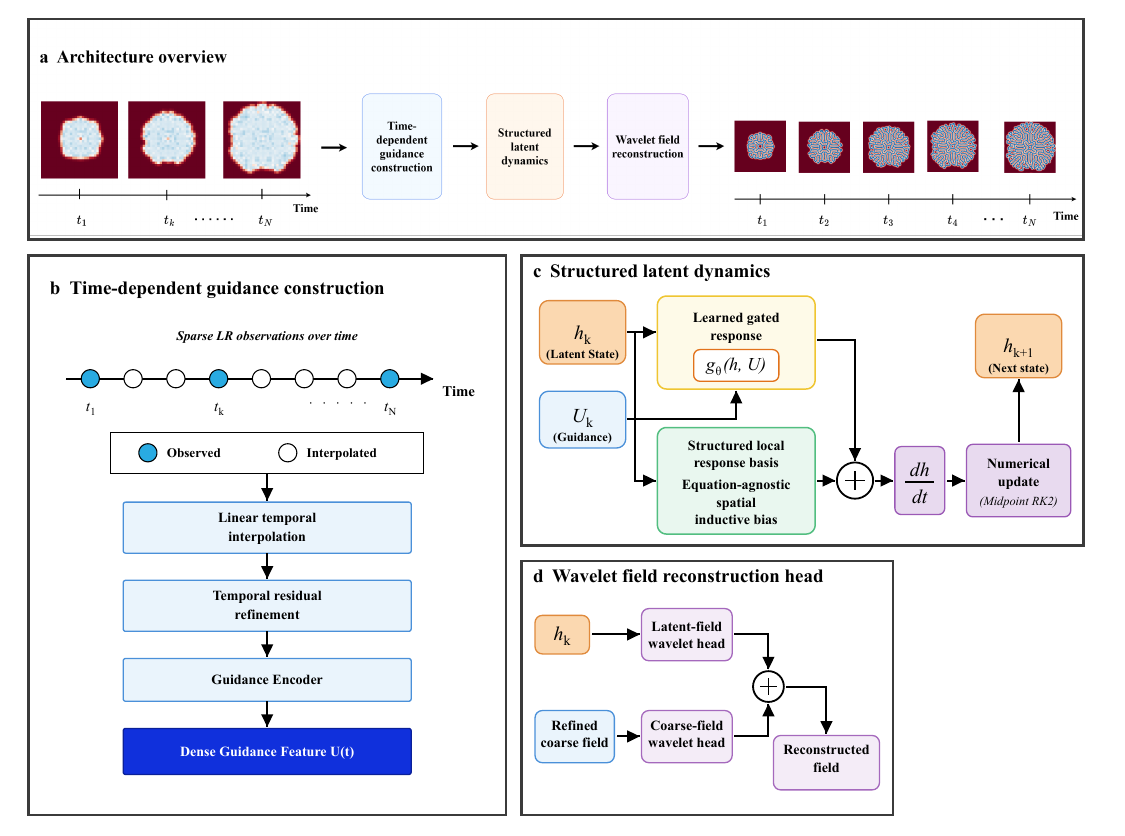}
\caption{
\textbf{Structured Latent Dynamics framework.}
\textbf{a,} A sequence of spatially coarse observations at sparse times is used to reconstruct a high-resolution field sequence. Information from the coarse sequence guides the reconstruction throughout the output interval.
\textbf{b,} Temporal refinement. Linear interpolation increases the number of time positions. A temporal convolutional module predicts a correction at each spatial location, giving a refined sequence on the coarse grid.
\textbf{c,} Structured latent dynamics. The hidden state is carried between output times. Its update combines a gated response computed from the state and the refined input with fixed horizontal and vertical central differences and a five-point Laplacian of the state. Learned channel-wise coefficients weight the three spatial responses.
\textbf{d,} Multiscale wavelet synthesis. Two learned branches process the hidden state and the refined coarse field, respectively. Each predicts a complete set of Haar wavelet coefficients and reconstructs a high-resolution field by inverse wavelet synthesis. Their outputs are added to form the prediction.
}
\label{fig:arch}
\end{figure}

\section{Results}
\label{sec:results}

\subsection{SLD reconstructs trajectories by updating a hidden state with coarse observations}
\label{subsec:sld_results}

SLD carries information between reconstructed frames through a latent state, a learned multichannel feature field on the coarse grid (Fig.~\ref{fig:arch}). Features extracted from the temporally refined coarse input guide each state update. Two Haar wavelet synthesis branches map the updated state and the refined coarse frame to the high-resolution output.

Given a low-resolution input segment, linear temporal interpolation followed by a learned correction produces a refined sequence with \(T_{\rm out}=100\) frames. A spatial encoder maps these frames to guidance features \(U_k\in\mathbb{R}^{C_h\times H\times W}\). Here, \(H\times W\) is the coarse spatial grid and \(C_h\) is the number of feature channels. The first guidance feature initializes the hidden state through a learned pointwise projection. SLD then updates the state \(h_k\in\mathbb{R}^{C_h\times H\times W}\) at each output index \(k\) by
\begin{equation}
    \begin{gathered}
    h_k
    =
    \Phi_{\Delta t}\!\left(h_{k-1},U_k;F_\theta\right),
    \qquad F_\theta(h,U)
    =
    G_\theta(h,U)+\mathcal{S}_{\psi}h ,
    \qquad k=1,\ldots,T_{\rm out}.
    \end{gathered}
    \label{eq:sld_latent_update}
\end{equation}
where \(\Phi_{\Delta t}\) advances the hidden state by one explicit midpoint step of size \(\Delta t\), holding \(U_k\) fixed within the step. We use \(\Delta t=1\) in network time units. The learned response \(G_\theta\) uses gates to control how the current state changes in response to the guidance. The linear operator \(\mathcal{S}_{\psi}\) adds local spatial differences of the state. Methods specifies the midpoint update and gate equations.

The linear spatial response is
\begin{equation}
    \mathcal{S}_{\psi}h
    =
    \alpha\odot \Delta h
    +
    \beta_x\odot D_x h
    +
    \beta_y\odot D_y h ,
    \label{eq:sld_structured_response}
\end{equation}

where \(D_xh\) and \(D_yh\) are horizontal and vertical central differences, and \(\Delta h\) is a five-point Laplacian. These fixed filters act separately on each latent channel. The learned vectors \(\psi=(\alpha,\beta_x,\beta_y)\) assign one coefficient per response and channel, with \(\odot\) denoting element-wise multiplication. The central differences measure horizontal and vertical variation. The five-point Laplacian subtracts four times the central value from the sum of its four nearest neighbours. Supplementary Note~\ref{supp:operator_characterization} gives the Fourier multiplier, operator norm bound and energy identity of this linear component.

Two Haar wavelet synthesis branches reconstruct the high-resolution output. One receives the evolved hidden state, and the other receives the temporally refined coarse field directly. Each branch predicts approximation coefficients at the coarsest scale and detail coefficients at all three synthesis levels. Three inverse Haar synthesis levels increase the spatial resolution by \(2\times\) per level, giving the target \(8\times\) refinement. The two reconstructed fields are added to form the prediction.

We next evaluate reconstruction on held-out trajectories and later windows, and then examine the effect of removing \(\mathcal{S}_{\psi}\) in a Gray--Scott ablation.

\subsection{Evaluation protocols test held-out trajectories and later temporal windows}
\label{subsec:evaluation_overview}

We evaluated SLD on two-dimensional Gray--Scott reaction--diffusion, complex Ginzburg--Landau and two-component viscous Burgers systems. Low-resolution inputs were constructed by downsampling the numerical reference trajectories. Each input contained 25 frames, each target contained 100 frames, and the target grid had eight times as many points in each spatial direction. The first protocol held out complete trajectories. In the second, the early segments \(S_1\)--\(S_3\) from all ten trajectories were used for parameter optimization. The disjoint later segments \(S_4\)--\(S_5\) were used for validation or testing, with their corresponding low-resolution inputs supplied to the models.

The two principal baselines had comparable numbers of trainable parameters and shared SLD's data splits, normalization, objective, optimizer settings and checkpoint selection rule. Unstructured Controlled Latent Dynamics (UCLD) retained the temporal refinement, guidance, midpoint update and reconstruction heads, and replaced the entire SLD latent vector field with a learned convolutional field. ConvLSTM retained the temporal refinement and reconstruction interface and used a discrete convolutional recurrent update.

We report relative \(L_2\) error and mean absolute error, together with radial power spectra, errors at individual output times and root mean square error (RMSE) grouped by the gradient magnitude of the reference field. These diagnostics measure spectral content, temporal variation in error and reconstruction accuracy from smooth to sharp regions. Each system contained ten independently generated trajectories. Methods and Supplementary Notes~\ref{supp:data_generation}--\ref{supp:numerical_results} provide the equations, numerical solvers, data splits, metric aggregation and training settings.

\subsection{SLD reconstructs held-out trajectories from unseen initial conditions}
\label{subsec:unseen_ic}

We first evaluated reconstruction on trajectories generated from unseen initial conditions. For each system, we used a \(7/2/1\) train/validation/test split by trajectory identity over ten folds, with each trajectory serving exactly once as the test trajectory. At inference, each model received low-resolution segments from the held-out trajectory and reconstructed their high-resolution counterparts.

Representative reconstructions are shown in Fig.~\ref{fig:caseB_visualization}. All three methods recovered the global field morphology. Relative to UCLD and ConvLSTM, SLD showed less smoothing and displacement along Gray--Scott interfaces, clearer spiral transitions in the magnified Ginzburg--Landau region, and less blurring near the displayed Burgers transitions.

The validation histories shown in Fig.~\ref{fig:caseB_visualization} also gave lower relative \(L_2\) errors for SLD than for both controls in all three systems. Fig.~\ref{fig:caseB_quant} evaluates the selected checkpoints using field errors, spectra and errors grouped by time and gradient magnitude.

SLD had the lowest mean relative \(L_2\) error and MAE for all three systems (Fig.~\ref{fig:caseB_quant}\textbf{a}). For each fold, errors were averaged over the test segments of its held-out trajectory. The reported mean and s.d. were then calculated across the ten held-out trajectories. The mean UCLD relative \(L_2\) errors were \(2.82\times\), \(2.39\times\) and \(2.92\times\) those of SLD on Gray--Scott, Ginzburg--Landau and Burgers, respectively. The corresponding ratios of ConvLSTM error to SLD error were \(2.45\times\), \(2.51\times\) and \(2.66\times\). Numerical values are reported in Supplementary Table~\ref{tab:taskA_numerical}.

The spectral and local error diagnostics showed the same ordering. In the representative Gray--Scott comparison, SLD had lower spectral relative error than both controls over most intermediate and high wavenumbers (Fig.~\ref{fig:caseB_quant}\textbf{b}). Its mean relative \(L_2\) error also remained lower in each of the four test windows ordered by their start times (Fig.~\ref{fig:caseB_quant}\textbf{c}). Across all three systems, SLD had lower RMSE in each group of points defined by the reference gradient magnitude, including the groups with the sharpest gradients (Fig.~\ref{fig:caseB_quant}\textbf{d}). These results show improved reconstruction of both global fields and regions with sharp spatial variation.

\begin{figure}[!htb]
\centering
   \begin{minipage}{0.9\textwidth}
     \centering
     \includegraphics[width=\linewidth]{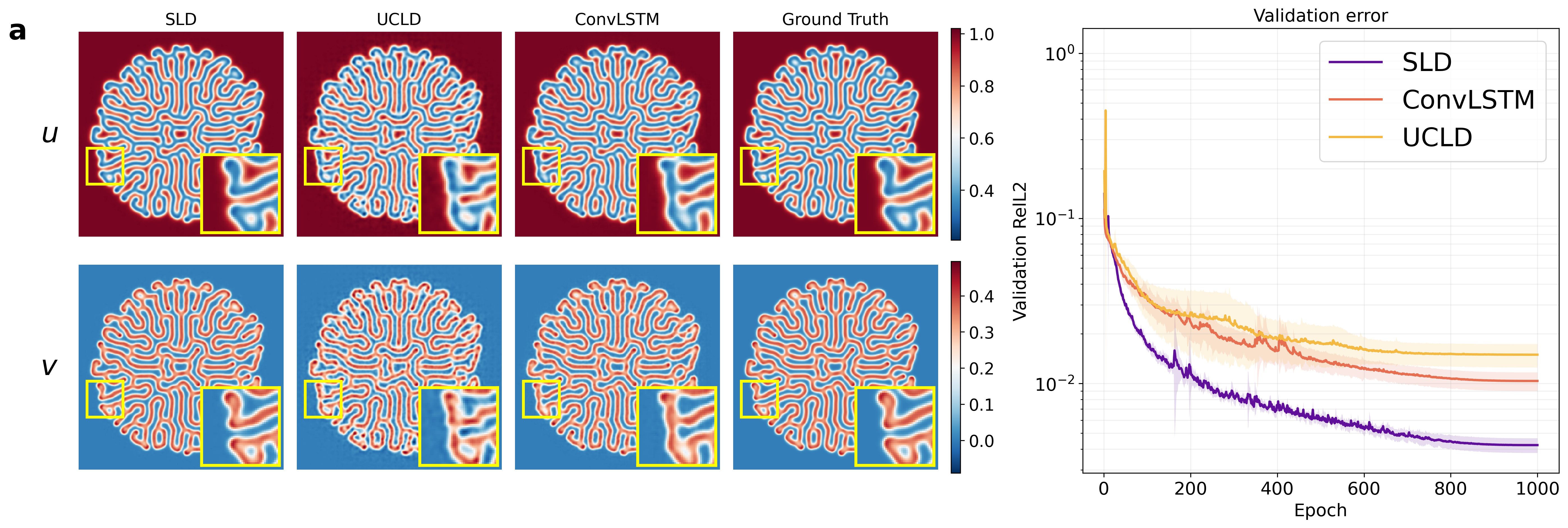}
   \end{minipage}

   \begin{minipage}{0.9\textwidth}
     \centering
     \includegraphics[width=\linewidth]{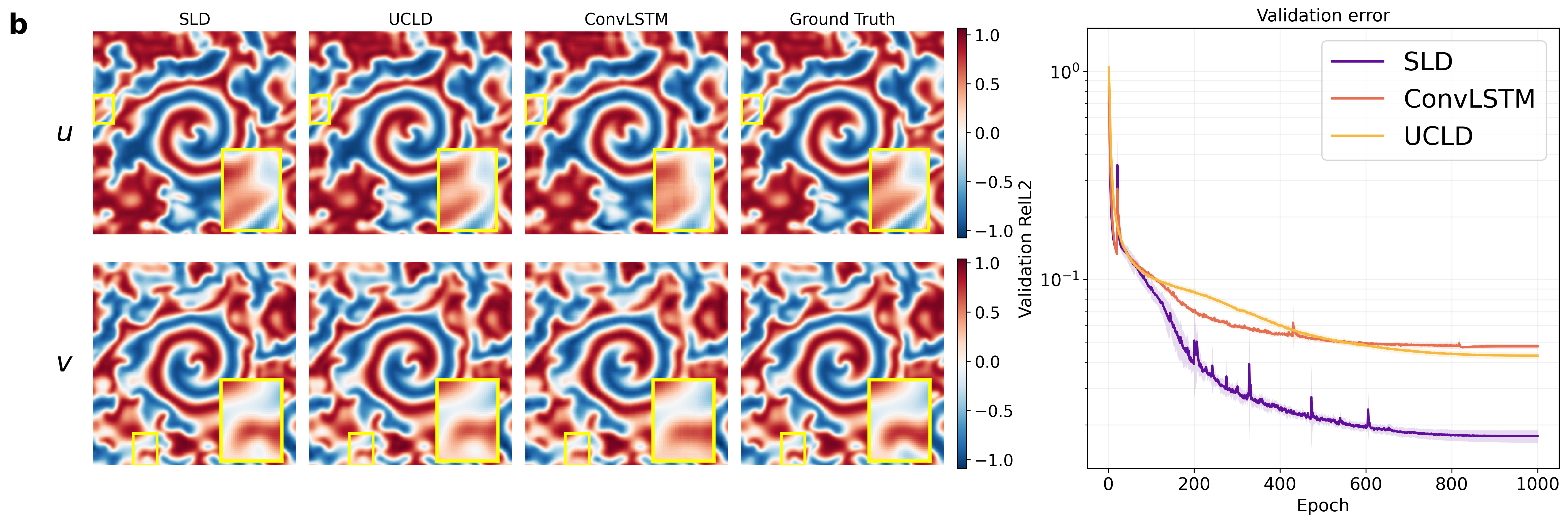}
   \end{minipage}

   \begin{minipage}{0.9\textwidth}
     \centering
     \includegraphics[width=\linewidth]{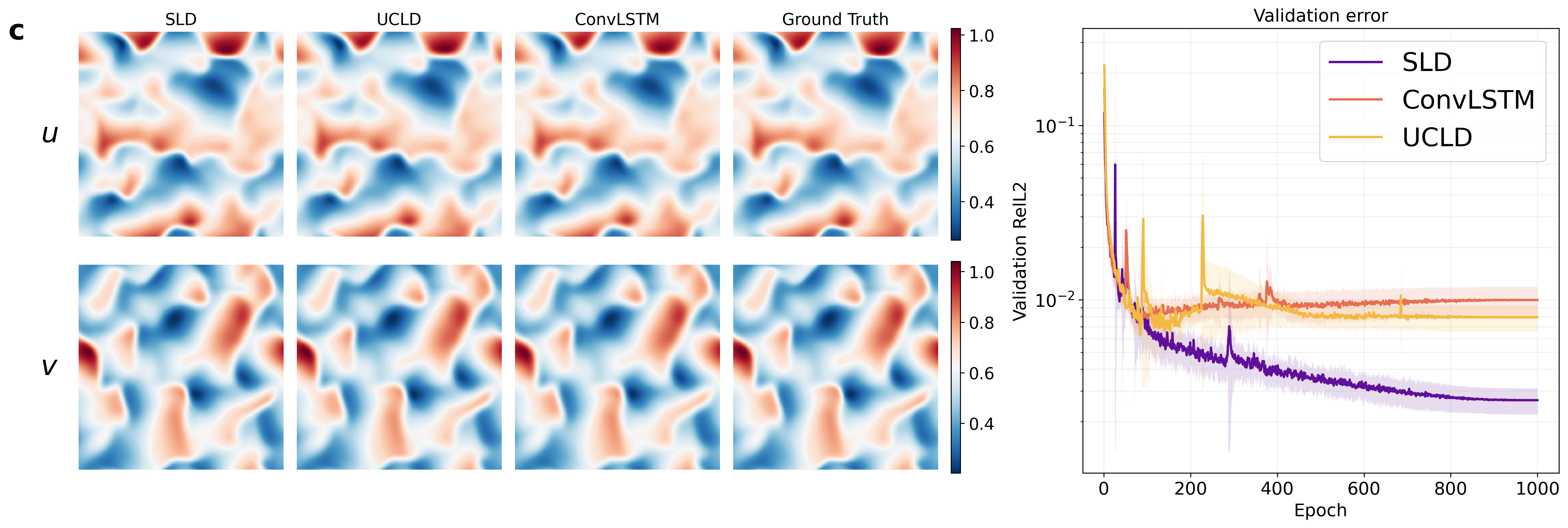}
   \end{minipage}

\caption{
\textbf{Reconstruction from unseen initial conditions.} \textbf{a,} Gray--Scott. \textbf{b,} Complex Ginzburg--Landau. \textbf{c,} Two-component viscous Burgers. For each system, high-resolution \(u\) and \(v\) reconstructions from SLD, UCLD and ConvLSTM are shown with the reference field. UCLD uses an unstructured convolutional latent vector field, and ConvLSTM uses a discrete recurrent update; both have parameter counts comparable to SLD. Yellow boxes mark the magnified regions. The right column shows validation relative \(L_2\) error histories under the shared training objective. Figure~\ref{fig:caseB_quant} reports the corresponding field errors, spectra and errors grouped by time and reference gradient magnitude.
}
\label{fig:caseB_visualization}
% AUTHOR CHECK REQUIRED [REV-06]: Define the center curve and shaded interval
% in the validation histories from the plotting code.
\end{figure}

\begin{figure}[!htb]
\centering
   \begin{minipage}{0.8\textwidth}
     \centering
     \includegraphics[width=\linewidth]{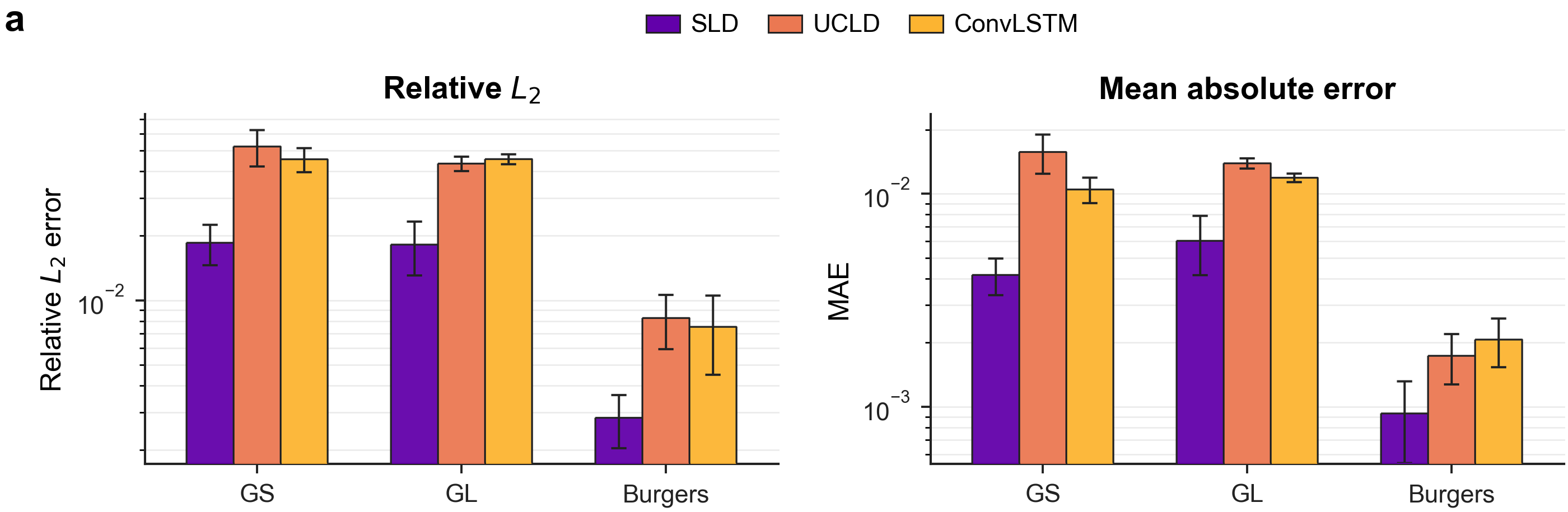}
   \end{minipage}
   \begin{minipage}{0.8\textwidth}
     \centering
     \includegraphics[width=\linewidth]{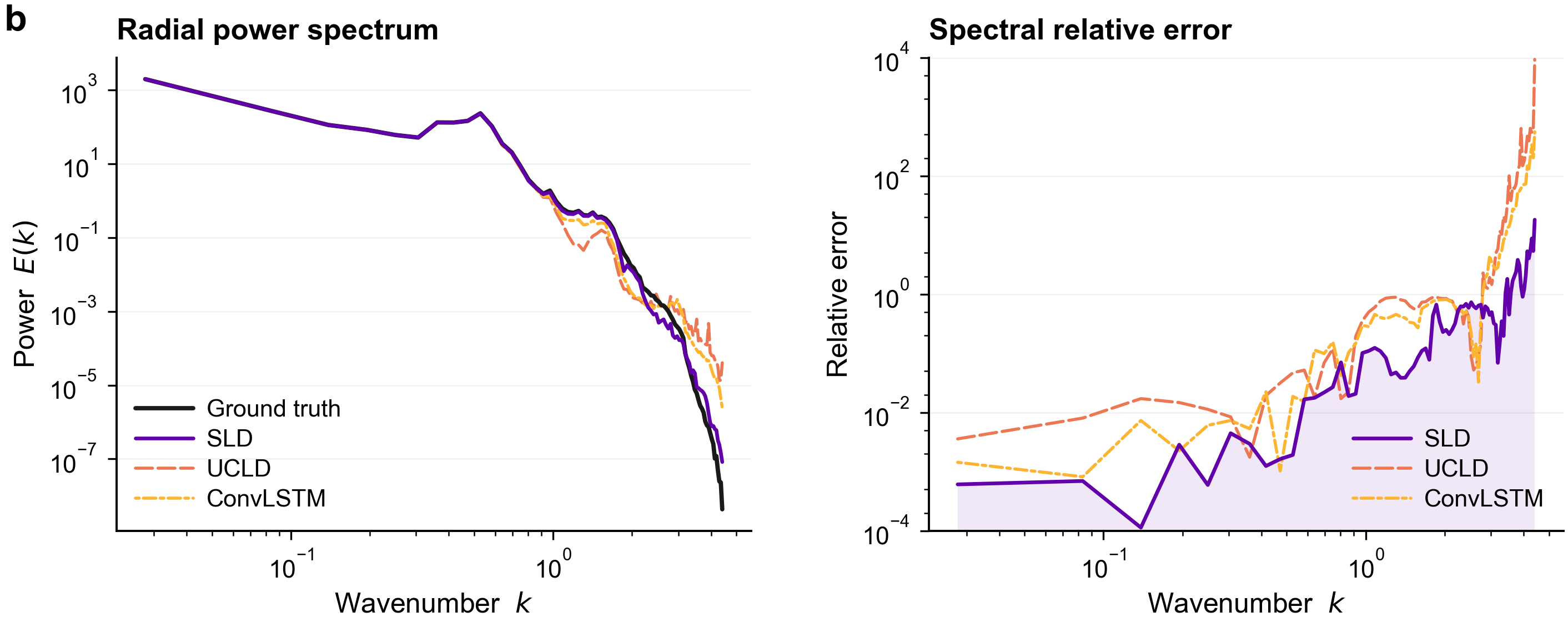}
   \end{minipage}

   \begin{minipage}{0.8\textwidth}
     \centering
     \includegraphics[width=\linewidth]{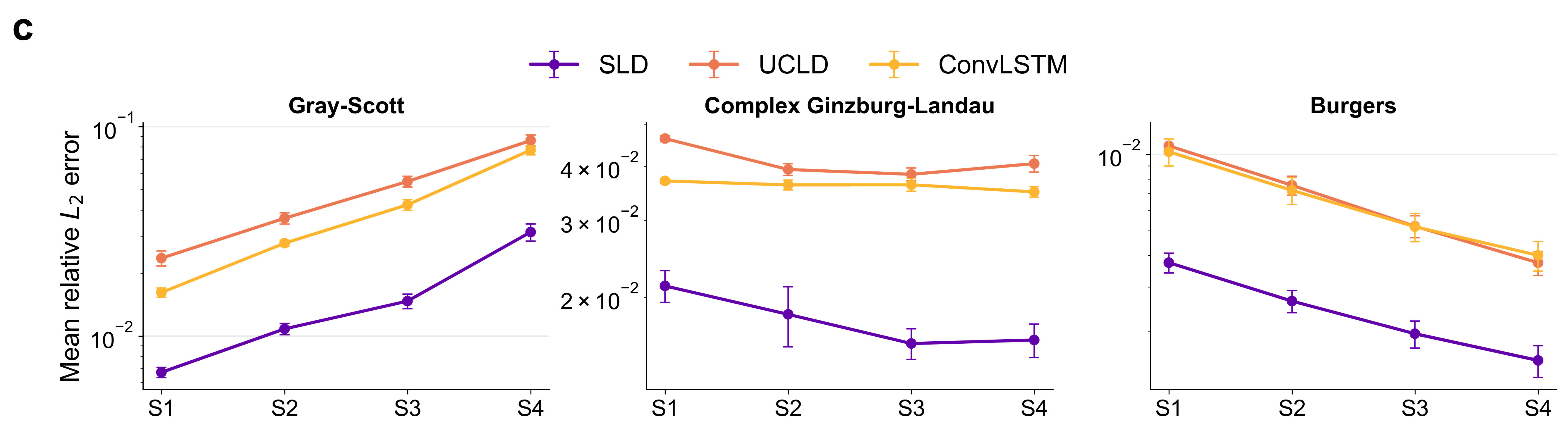}
   \end{minipage}
   \begin{minipage}{0.8\textwidth}
     \centering
     \includegraphics[width=\linewidth]{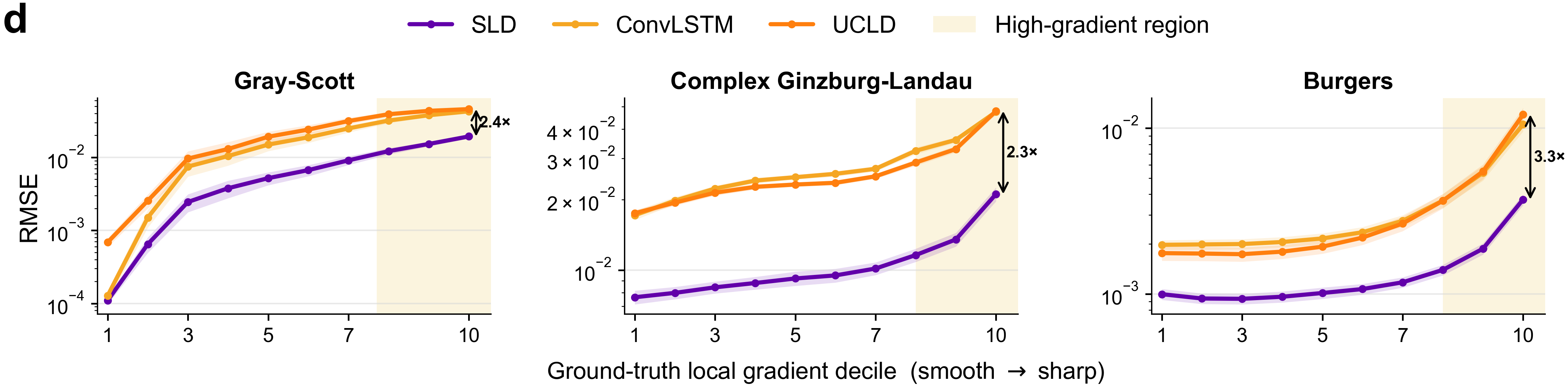}
   \end{minipage}
\caption{
\textbf{Quantitative evaluation under unseen initial conditions.} \textbf{a,} Relative \(L_2\) error and mean absolute error (MAE) for Gray--Scott (GS), complex Ginzburg--Landau (GL) and Burgers. Errors are first averaged over the test segments of each held-out trajectory. Bars and error bars show the mean and standard deviation across the ten held-out trajectories (\(n=10\)). \textbf{b,} Representative Gray--Scott comparison of the radial power spectrum \(E(k)\) and its relative error at wavenumber \(k\). These quantities measure the recovery of spatial variation at different scales. \textbf{c,} Mean relative \(L_2\) error for four test windows beginning at low-resolution indices 100, 150, 200 and 250. \textbf{d,} RMSE in ten groups of space--time points ordered by the local gradient magnitude of the reference field. The groups run from smooth to sharp regions, and shading marks the region with large gradients. SLD is compared with UCLD and ConvLSTM, which have comparable parameter counts.
}
\label{fig:caseB_quant}
% AUTHOR CHECK REQUIRED [REV-07]: In panel c, replace S1--S4 by start indices
% 100,150,200,250 to avoid conflict with S4=225 in the other protocol. Define
% panel-c error bars, panel-d ratio arrows, and panel-b wavenumber units.
\end{figure}

\subsection{SLD reconstructs later high-resolution windows from coarse trajectory inputs}
\label{subsec:later_time}

We next tested reconstruction in later windows of the same ten trajectories (Fig.~\ref{fig:later_time}\textbf{a}). Early segments \(S_1\)--\(S_3\) from all ten trajectories supplied high-resolution targets for parameter optimization. The later segments \(S_4\)--\(S_5\) occupied disjoint time intervals. For each of five random partitions, later high-resolution fields from five trajectories were used for checkpoint selection, and those from the other five were used for testing. At inference, each model received the low-resolution input from the later segment being reconstructed. This protocol evaluates reconstruction at later times for trajectories whose early evolution was used in training.

SLD had the lowest mean relative \(L_2\) error and MAE in the later test windows for all three systems (Fig.~\ref{fig:later_time}\textbf{b}). Within each partition, errors were averaged over ten test segments: two later windows from each of five trajectories. The reported mean and s.d. were calculated across the five resulting test summaries (\(n=5\) repeated partitions). The mean UCLD relative \(L_2\) errors were \(5.43\times\), \(2.96\times\) and \(6.88\times\) those of SLD on Gray--Scott, Ginzburg--Landau and Burgers, respectively. The corresponding ratios of ConvLSTM error to SLD error were \(2.30\times\), \(3.66\times\) and \(6.04\times\). Aggregate values are reported in Supplementary Table~\ref{tab:taskB_numerical}.

Figure~\ref{fig:later_time}\textbf{c} shows error at each output time after concatenating the \(S_4\) and \(S_5\) curves and omitting the first ten high-resolution frames of each segment. SLD remained below both controls throughout the plotted Gray--Scott and Ginzburg--Landau intervals and over most of the Burgers interval. Its lower error persisted into the farther window \(S_5\), showing that the model reconstructed later fine-scale evolution from the coarse observations supplied in those windows.

\begin{figure}[!htb]
\centering
   \begin{minipage}{0.8\textwidth}
     \centering
     \includegraphics[width=\linewidth]{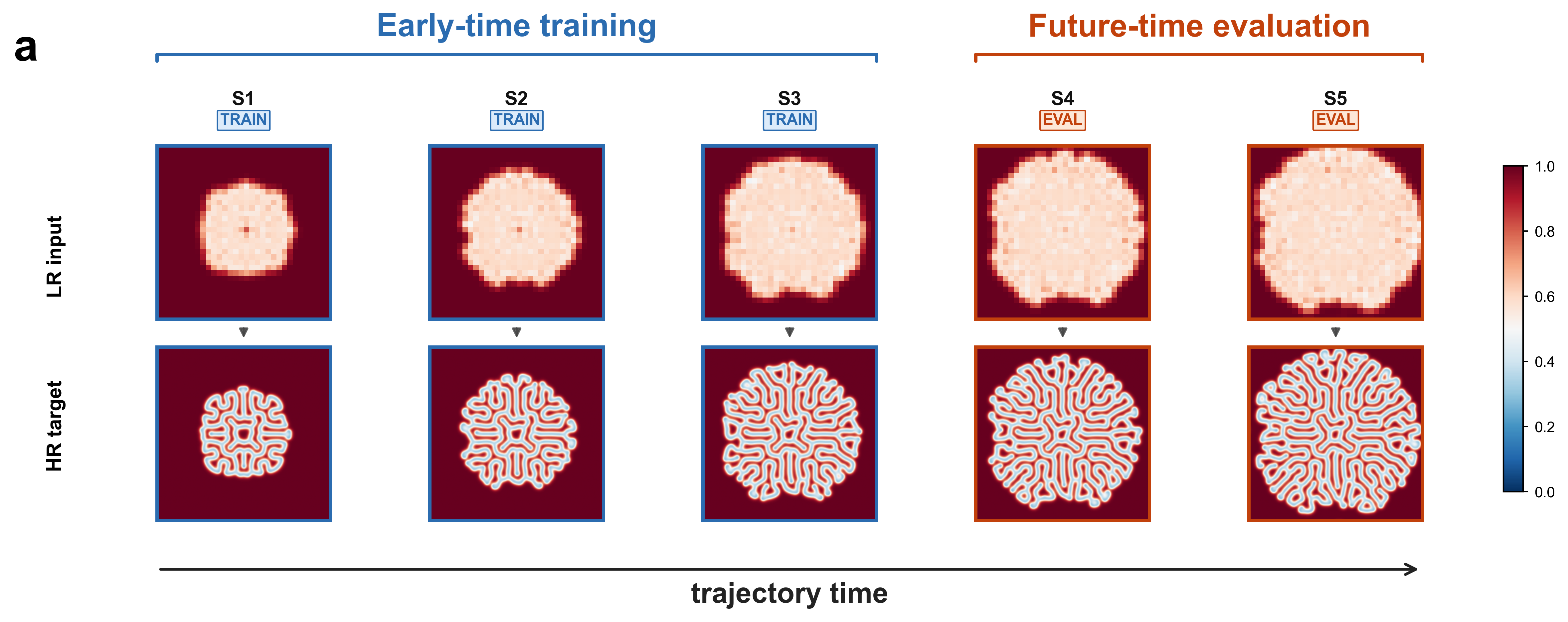}
   \end{minipage}

   \vspace{0.5em}

   \begin{minipage}{0.8\textwidth}
     \centering
     \includegraphics[width=\linewidth]{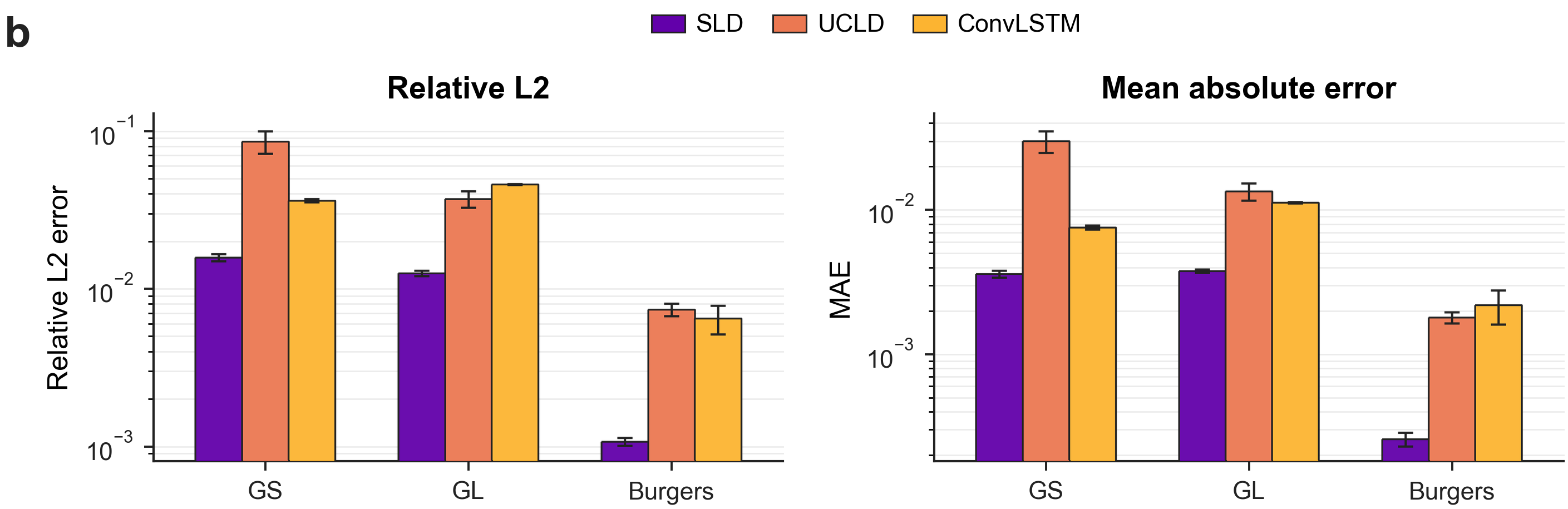}
   \end{minipage}

   \vspace{0.5em}

   \begin{minipage}{0.8\textwidth}
     \centering
     \includegraphics[width=\linewidth]{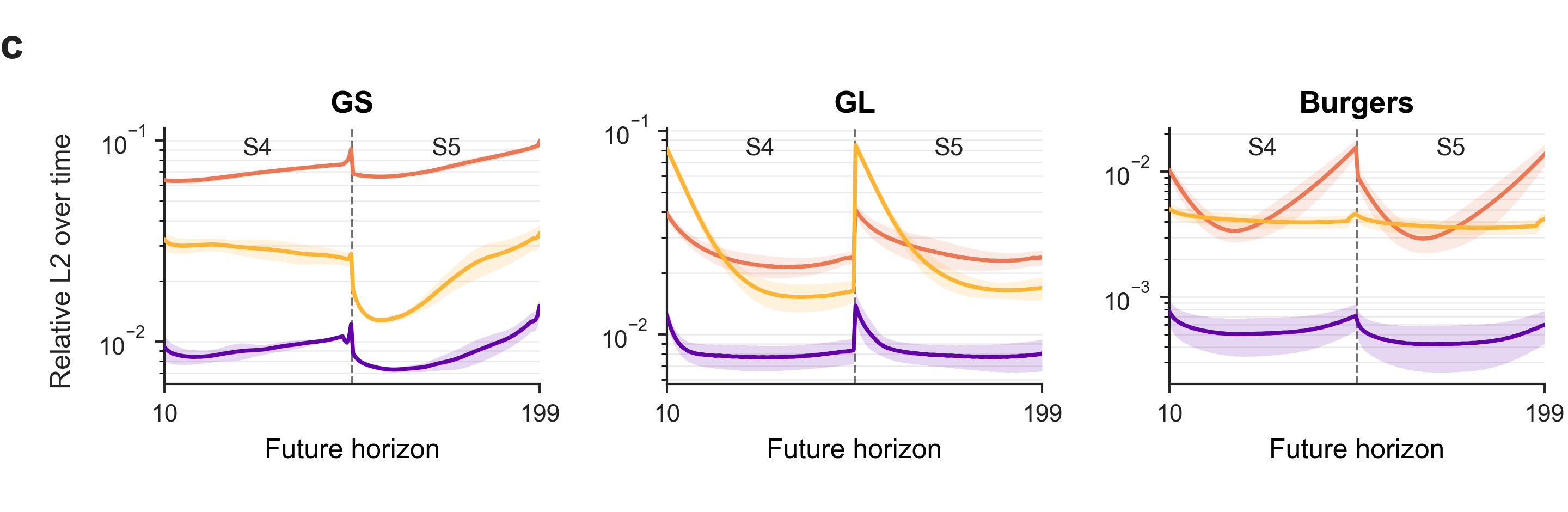}
   \end{minipage}

\caption{\textbf{Reconstruction in later temporal windows.} \textbf{a,} Early segments \(S_1\)--\(S_3\) from all ten trajectories provide high-resolution targets for parameter optimization. Later segments \(S_4\)--\(S_5\) occupy disjoint time intervals. In each partition, later targets from five trajectories select the checkpoint and those from the other five evaluate the selected model. All models receive the low-resolution inputs of the later segments being reconstructed. \textbf{b,} Relative \(L_2\) error and mean absolute error for Gray--Scott (GS), complex Ginzburg--Landau (GL) and Burgers. Within each partition, errors are averaged over the two later segments from each of five test trajectories. Bars and error bars show the mean and s.d. of these summaries across five random partitions (\(n=5\)). \textbf{c,} Relative \(L_2\) error at each output time, with the \(S_4\) and \(S_5\) curves concatenated after the first ten high-resolution frames of each segment are omitted. Dashed vertical lines separate \(S_4\) and \(S_5\). Shading shows the s.d. across the five curves, each obtained from one partition. UCLD and ConvLSTM have parameter counts comparable to SLD.
}
\label{fig:later_time}
% AUTHOR CHECK REQUIRED [REV-08]: Replace the embedded "Future horizon" axis
% label by the actual output index or physical-time convention. State why the
% first ten frames of each segment are omitted and which metrics use this cut.
\end{figure}

\subsection{Gray--Scott comparisons test additional architectures and the spatial response}
\label{subsec:local_response_ablation}

Under the unseen-initial-condition protocol described in Section~\ref{subsec:unseen_ic}, we compared SLD on Gray--Scott with implementations of FNO~\cite{li2021fourier}, LNO~\cite{cao2024laplace}, MambaNO~\cite{zheng2024MambaNO}, OSEDiff~\cite{wu2024onestep} and PhySR~\cite{Ren2023PhySR} adapted to the reconstruction task. We also trained an SLD variant with \(\mathcal{S}_{\psi}=0\), retaining the gated response, the remaining architecture and the training objective. All models used the same trajectory splits, input and output resolutions, normalization and evaluation metrics. Supplementary Note~\ref{supp:additional_baselines} describes the adaptations and hyperparameters. The narrow labyrinthine interfaces in Gray--Scott provide a representative test of spatial detail lost through downsampling.

In the representative fields, SLD without \(\mathcal{S}_{\psi}\) retained the global pattern layout but produced thicker and more displaced interfaces than full SLD (Fig.~\ref{fig:ablation}\textbf{a}). The additional baselines also showed greater smoothing, displacement or artificial high-frequency texture in the displayed interface regions.

Full SLD had the lowest mean relative \(L_2\) error and MAE among the evaluated models (Fig.~\ref{fig:ablation}\textbf{b}). Removing the structured spatial term increased the mean relative \(L_2\) error from \(1.855\times10^{-2}\) to \(3.720\times10^{-2}\) and the MAE from \(4.159\times10^{-3}\) to \(7.824\times10^{-3}\). This ablation supports the contribution of the structured spatial term to Gray--Scott reconstruction.

\begin{figure}[!htb]
\centering
   \begin{minipage}{0.8\textwidth}
     \centering
     \includegraphics[width=\linewidth]{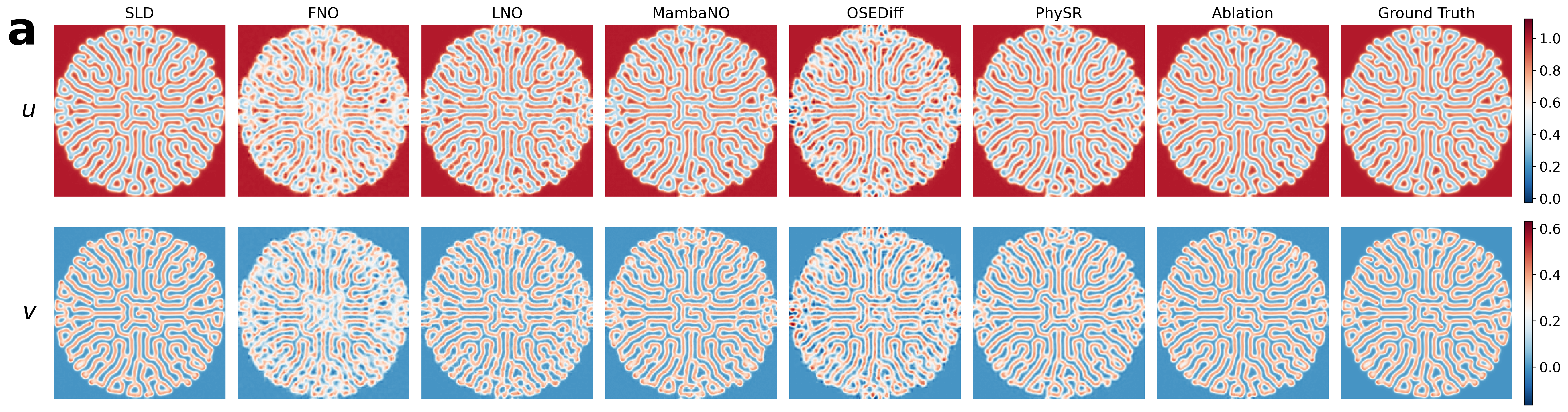}
   \end{minipage}

   \vspace{0.5em}

   \begin{minipage}{0.8\textwidth}
     \centering
     \includegraphics[width=\linewidth]{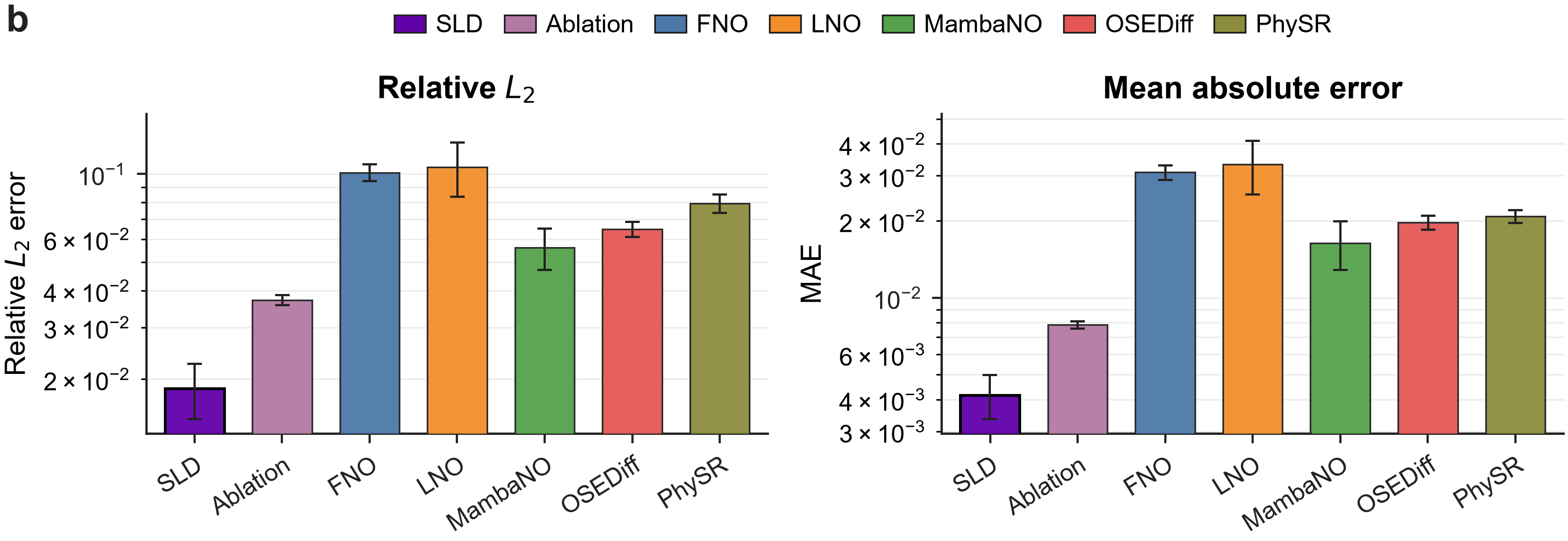}
   \end{minipage}

\caption{\textbf{Additional Gray--Scott comparisons and ablation of the spatial response.} \textbf{a,} Representative reconstructions of \(u\) and \(v\) from an unseen initial condition. Full SLD is compared with SLD trained with \(\mathcal{S}_{\psi}=0\) and the additional baselines adapted to this task, as described in Supplementary Note~\ref{supp:additional_baselines}. \textbf{b,} Relative \(L_2\) error and mean absolute error under the same protocol for held-out trajectories. Bars and error bars show the mean and standard deviation across ten trajectories. The ablation retains the gated response, the remaining architecture and the training objective. Lower values indicate better reconstruction.
}
\label{fig:ablation}
\end{figure}

\section{Discussion}
\label{sec:discussion}

SLD reconstructs high-resolution PDE trajectories by using coarse observations throughout the output interval. A hidden state carries information between reconstructed frames, and the refined coarse sequence guides each state update. Across the three two-dimensional systems, SLD had lower mean errors than UCLD and ConvLSTM with comparable parameter counts. This result held both for held-out trajectories and for later windows of trajectories whose early segments were used for training.

The architecture combines temporal refinement, a learned nonlinear latent response, fixed spatial differences and wavelet synthesis. UCLD and ConvLSTM provide comparisons with alternative rules for updating the hidden state, while the Gray--Scott ablation directly tests the spatial differences within SLD. Removing the structured spatial term increased reconstruction errors on Gray--Scott. This term combines central differences and a Laplacian of the latent state, with channel-wise coefficients learned during training.

Beyond aggregate field errors, SLD improved reconstruction in regions with large gradients across all three systems and recovered spectral content more accurately in the representative Gray--Scott comparison. The time-dependent error curves also showed lower errors through most or all of the evaluated later intervals. Together, these diagnostics describe improvements in the spatial detail and temporal evolution of the reconstructed fields.

The present experiments use ten noise-free numerical trajectories per system on two-dimensional periodic grids, with fixed refinement factors and coarse inputs obtained by downsampling. Additional baseline comparisons and the ablation of the spatial response are evaluated on Gray--Scott. Future work can extend this setting to observations from independent simulations on coarse grids, noisy or irregular measurements, varying parameters and boundary conditions, and three-dimensional domains. Evaluating later windows of entirely unseen trajectories would further test temporal and trajectory generalization together.

\section{Methods}
\label{sec:methods}

\subsection{Problem formulation and segment construction}
\label{subsec:problem_setting}

We reconstruct high-resolution PDE trajectories from observations sampled on a coarse spatial grid at sparse times. Let \(\mathbf{x}_{1:T_{\rm in}}\in\mathbb{R}^{T_{\rm in}\times C\times H\times W}\) denote an input segment and \(\mathbf{y}_{1:T_{\rm out}}\in\mathbb{R}^{T_{\rm out}\times C\times sH\times sW}\) its target segment. Here, \(C\) is the number of physical variables, \(H\) and \(W\) are the coarse grid dimensions, \(s\) is the spatial refinement factor, and \(r=T_{\rm out}/T_{\rm in}\) is the ratio of sequence lengths. All experiments use \(C=2\), \(H=W=32\), \(s=8\), \(T_{\rm in}=25\), and \(T_{\rm out}=100\). Each model therefore maps 25 two-channel \(32\times32\) frames to 100 two-channel \(256\times256\) frames.

The objective is to learn a reconstruction map
\begin{equation}
    \mathcal{M}_\theta:
    \mathbf{x}_{1:T_{\rm in}}
    \mapsto
    \widehat{\mathbf{y}}_{1:T_{\rm out}},
    \label{eq:reconstruction_map}
\end{equation}

that predicts the high-resolution trajectory segment corresponding to the coarse input sequence. Reconstruction fidelity is evaluated using aggregate field errors, spectral diagnostics, errors at individual output times and errors grouped by reference gradient magnitude, as defined below.

For each PDE system, we generate high-resolution trajectories from distinct initial conditions and downsample them in space and time to obtain the corresponding low-resolution observations. We denote this process by
\begin{equation}
    \mathbf{x}
    =
    \mathcal{D}_{t,r}\mathcal{D}_{s}
    (\mathbf{y}),
    \label{eq:downsampling}
\end{equation}

where \(\mathcal{D}_{s}\) denotes spatial downsampling and \(\mathcal{D}_{t,r}\) keeps every fourth stored frame. A paired input samples target offsets \(0,4,\ldots,96\), while its 100-frame target covers offsets \(0,\ldots,99\). The temporal module converts these 25 observations into 100 features used to guide reconstruction. Supplementary Note~\ref{supp:temporal_refinement_details} specifies the interpolation coordinates. Segment start indices are fixed before training and shared by all models. In the protocol for reconstruction at later times, \(S_1\)--\(S_3\) provide training targets, and \(S_4\)--\(S_5\) provide validation and test targets. Supplementary Tables~\ref{tab:segment_times} and~\ref{tab:splits} give the indices and assignments.

Low- and high-resolution fields are normalized channel-wise using their respective statistics, computed only from the training data in each split. Within a split, SLD and all baselines use the same low- and high-resolution normalization constants. Circular padding and periodic finite differences implement the periodic boundary convention used in the numerical simulations.

\subsection{Structured Latent Dynamics}
\label{subsec:sld_methods}

SLD first refines the coarse input sequence in time and encodes each refined frame into guidance features. These features guide sequential updates of the latent state on the coarse grid. At each output index, two Haar wavelet synthesis branches decode the updated state and the refined coarse frame, and their outputs are added to form the prediction.

\paragraph{Temporal guidance construction.}
We first linearly resample the normalized low-resolution input sequence from 25 to 100 positions, then apply a learned temporal correction:
\begin{equation}
    \widetilde{\mathbf{x}}_{1:T_{\rm out}}
    =
    \mathcal{T}_\theta(\mathbf{x}_{1:T_{\rm in}}).
    \label{eq:temporal_refinement}
\end{equation}

At each spatial location, one-dimensional temporal convolutions predict a correction that is added to the linearly interpolated sequence. A spatial encoder then maps each refined low-resolution frame to a guidance feature,
\begin{equation}
    \mathbf{U}_k
    =
    E_\theta(\widetilde{\mathbf{x}}_k),
    \qquad
    k=1,\ldots,T_{\rm out},
    \label{eq:guidance_encoding}
\end{equation}

where \(\mathbf{U}_k\in\mathbb{R}^{C_h\times H\times W}\). The first guidance feature initializes the latent state,
\begin{equation}
    \mathbf{h}_0
    =
    P_\theta(\mathbf{U}_1),
    \label{eq:initial_latent}
\end{equation}
where \(P_\theta\) is a \(1\times1\) convolution that mixes feature channels at each grid point. The state \(\mathbf h_0\) initializes the recurrence; the first output is decoded from \(\mathbf h_1\) after the first update.

\paragraph{Guided latent evolution.}

The latent state \(\mathbf{h}_k\in\mathbb{R}^{C_h\times H\times W}\) is updated on the coarse grid, with \(C_h=64\). The index \(k\) counts reconstruction updates. We use a unit step \(\Delta t=1\) in network time; the physical time intervals of the simulated PDE trajectories are specified separately in Supplementary Note~\ref{supp:data_generation}. The vector field is written as
\begin{equation}
    \frac{{\rm d}\mathbf{h}}{{\rm d}t}
    =
    F_\theta(\mathbf{h}(t),\mathbf{U}(t)),
    \label{eq:latent_ode}
\end{equation}

where \(t\) denotes network time and \(\mathbf U(t)\) supplies the guidance. Each reconstruction step uses the following explicit midpoint update, with the same \(\mathbf U_k\) in both vector-field evaluations:
\begin{align}
    \mathbf{k}_1
    &=
    F_\theta(\mathbf{h}_{k-1},\mathbf{U}_k),\\
    \mathbf{k}_2
    &=
    F_\theta\!\left(
    \mathbf{h}_{k-1}
    +
    \frac{1}{2}\Delta t\,\mathbf{k}_1,
    \mathbf{U}_k
    \right),\\
    \mathbf{h}_k
    &=
    \mathbf{h}_{k-1}
    +
    \Delta t\,\mathbf{k}_2,
    \qquad
    k=1,\ldots,T_{\rm out}.
    \label{eq:midpoint_update}
\end{align}

The two evaluations use different latent states and the same guidance feature. The resulting state is decoded as output frame \(k\).

The latent vector field is written as
\begin{equation}
    F_\theta(\mathbf{h}_k,\mathbf{U}_k)
    =
    G_\theta(\mathbf{h}_k,\mathbf{U}_k)
    +
    \mathcal{S}_{\psi}(\mathbf{h}_k),
    \label{eq:sld_vector_field}
\end{equation}
where \(G_\theta\) is a learned gated response and \(\mathcal{S}_{\psi}\) is the structured latent response.

\paragraph{Learned gated response.}
Given the latent state \(\mathbf h\) and guidance features \(\mathbf U\), the network computes a reset gate \(\mathbf r\) and an update gate \(\mathbf z\), each with entries between zero and one. The reset gate weights the state used to construct the candidate \(\widetilde{\mathbf h}\). The factor \(1-\mathbf z\) scales the response from the current state towards this candidate:
\begin{align}
    [\mathbf{r},\mathbf{z}]
    &=
    \sigma\!\left(
    \mathcal{G}_\theta([\mathbf{h},\mathbf{U}])
    \right),\\
    \widetilde{\mathbf{h}}
    &=
    \mathcal{C}_\theta(
    [\mathbf{r}\odot\mathbf{h},\mathbf{U}]
    ),\\
    G_\theta(\mathbf{h},\mathbf{U})
    &=
    (1-\mathbf{z})\odot
    (\widetilde{\mathbf{h}}-\mathbf{h}).
    \label{eq:gated_response}
\end{align}

Here, \([\cdot,\cdot]\) denotes concatenation along the channel dimension, \(\sigma\) is the sigmoid function, and \(\odot\) denotes element-wise multiplication. Circular padding wraps spatial convolution inputs across opposite domain boundaries. The maps \(\mathcal G_\theta\) and \(\mathcal C_\theta\) produce the gates and candidate state, respectively. When an entry of \(\mathbf z\) approaches one, the corresponding entry of the learned response approaches zero; when it approaches zero, the response approaches \(\widetilde{\mathbf h}-\mathbf h\).

\paragraph{Structured spatial operator.}

The structured spatial operator applies three fixed finite-difference filters independently to each latent channel:
\begin{equation}
    \mathcal{S}_{\psi}(\mathbf{h})
    =
    \boldsymbol{\alpha}\odot\Delta\mathbf{h}
    +
    \boldsymbol{\beta}_x\odot D_x\mathbf{h}
    +
    \boldsymbol{\beta}_y\odot D_y\mathbf{h}.
    \label{eq:structured_response_methods}
\end{equation}

Here, \(D_x\) and \(D_y\) compute horizontal and vertical central differences, and \(\Delta\) is the five-point Laplacian. They act on a latent grid with unit spacing. The vectors \(\psi=\{\boldsymbol\alpha,\boldsymbol\beta_x,\boldsymbol\beta_y\}\) contain one learned coefficient per channel for each filter; each coefficient is shared across spatial locations and reconstruction steps. The central differences measure directional changes, while the Laplacian compares a latent value with its four nearest neighbours. The filter definitions are shared across PDE systems, and their coefficients are fitted separately in each training run.

A Fourier-domain characterization of \(\mathcal{S}_{\psi}\), including its spectral multiplier, operator norm bound and energy identity on periodic grids, is provided in Supplementary Note~\ref{supp:operator_characterization}.

\paragraph{Multiscale wavelet synthesis.}

Two Haar wavelet synthesis branches produce the normalized high-resolution prediction:
\begin{equation}
    \widehat{\mathbf{y}}^{\rm norm}_k
    =
    W_\theta(\mathbf{h}_k)
    +
    W_{\rm skip}(\widetilde{\mathbf{x}}_k),
    \qquad
    k=1,\ldots,T_{\rm out}.
    \label{eq:wavelet_synthesis}
\end{equation}

The latent branch \(W_\theta\) maps the evolved state to a high-resolution field. The skip branch \(W_{\rm skip}\) maps the refined coarse input directly to a field of the same shape. Each branch predicts both approximation and detail coefficients for a Haar wavelet pyramid. Three synthesis stages double the spatial dimensions successively from \(32\times32\) to \(64\times64\), \(128\times128\), and \(256\times256\). The two branch outputs are added, and the sum is converted to physical units using the high-resolution training statistics. Supplementary Note~\ref{supp:wavelet_details} specifies the coefficient arrangement and synthesis formulas.

\subsection{Training objective}
\label{subsec:training_objective}

For each PDE system and evaluation protocol, SLD, UCLD and ConvLSTM use a common objective. The weight \(\lambda_{\rm phy}\) controls the contribution of the PDE residual penalty, with \(\lambda_{\rm phy}=0\) giving the reconstruction loss alone:
\begin{equation}
    \mathcal{L}
    =
    \mathcal{L}_{\rm data}
    +
    \lambda_{\rm phy}
    \mathcal{L}_{\rm phy}.
    \label{eq:total_loss}
\end{equation}
The supervised term is
\begin{equation}
    \mathcal{L}_{\rm data}
    =
    \frac{1}{T_{\rm out}N_{\rm frame}}
    \sum_{k=1}^{T_{\rm out}}
    \left\|
    \widehat{\mathbf{y}}_k
    -
    \mathbf{y}_k
    \right\|_1 ,
    \label{eq:data_loss}
\end{equation}
where \(N_{\rm frame}=C(sH)(sW)\) is the number of scalar values in one output frame and \(\|\cdot\|_1\) is the sum of absolute values. Both loss terms are evaluated after transforming predictions to physical units. The discrete residual \(\mathcal R_{\mathcal P}\) measures the discrepancy between the time derivative of a reconstructed field and the right-hand side of its governing PDE:
\begin{equation}
    \mathcal{L}_{\rm phy}
    =
    \frac{1}{|\mathcal{I}_{\rm phy}|}
    \sum_{k\in\mathcal{I}_{\rm phy}}
    \left\|
    \mathcal{R}_{\mathcal{P}}
    (\widehat{\mathbf{y}})_k
    \right\|_2^2 ,
    \label{eq:physics_loss}
\end{equation}

where \(\mathcal P\) denotes the PDE system and \(\mathcal I_{\rm phy}=\{1,\ldots,T_{\rm out}\}\) is the set of output indices used for residual evaluation. Supplementary Note~\ref{supp:training_details} gives the schedule for \(\lambda_{\rm phy}\), optimizer settings and checkpoint selection.
% AUTHOR CHECK REQUIRED [REV-01]: Specify the actual discrete PDE residuals,
% temporal endpoint treatment and reduction of the squared residual for every
% system. These are not supplied by the simulation equations or optimizer schedule.
\subsection{Evaluation protocols}
\label{subsec:evaluation_protocols}

The evaluation uses trajectory identity and temporal location to define the
train--validation--test splits.

\paragraph{Protocol for unseen initial conditions.}

For each PDE system, the ten trajectories are partitioned by identity into seven training, two validation and one test trajectory. Model parameters are fitted on segments from the seven training trajectories, and the checkpoint is selected using the two validation trajectories. The model then receives low-resolution segments from the held-out test trajectory and reconstructs their corresponding high-resolution targets. We repeat this procedure ten times so that every trajectory serves exactly once as the test trajectory. Within each fold, errors for individual segments are averaged for that held-out trajectory. The reported mean and s.d. are calculated from one test value per held-out trajectory (\(n=10\)).

\paragraph{Protocol for reconstruction at later times.}

For reconstruction at later times, all ten trajectory identities contribute targets from the early segments \(S_1\)--\(S_3\) to training. The subsequent segments \(S_4\)--\(S_5\) occupy disjoint time intervals. In each of five repeated partitions, five identities supply the later validation targets for checkpoint selection, and the remaining five supply the later test targets. At inference, each model receives the low-resolution input from the window being reconstructed. Thus, training and evaluation use different time windows of the same collection of trajectories.

Within each protocol and split, SLD and all compared models use identical trajectory assignments, segment start indices, normalization constants and checkpoint selection rules. Supplementary Table~\ref{tab:splits} summarizes the split structure.
% AUTHOR CHECK REQUIRED [REV-02]: Include the exact trajectory assignments and
% training seeds for the ten folds and five repeated partitions in the release.

\subsection{Baselines}
\label{subsec:baselines}

We compare SLD with two controlled baselines that test complementary rules for updating the state. Both use the same input and output resolutions, temporal and spatial refinement factors, normalization, periodic boundary treatment, wavelet reconstruction interface and training objective as SLD.

\paragraph{Unstructured Controlled Latent Dynamics.}
UCLD uses a convolutional vector field to update a latent state using the coarse guidance. It shares SLD's temporal refinement, guidance construction, state initialization, midpoint update, Haar wavelet synthesis and skip branch; its width is chosen to give a similar parameter count. The latent vector field is replaced by a fully learned convolutional vector field conditioned on the same time-dependent guidance:
\begin{equation}
    F_{\rm UCLD}(\mathbf{h}_k,\mathbf{U}_k)
    =
    \gamma\,
    \mathcal{N}_{\phi}
    ([\mathbf{h}_k,\mathbf{U}_k]),
    \label{eq:ucld_field}
\end{equation}

where \(\mathcal N_\phi\) is a stack of circularly padded convolutional blocks and \(\gamma\) is a learned scalar. This comparison changes the complete latent vector field, replacing both the gated response and the structured spatial operator. The separate ablation sets \(\mathcal S_\psi=0\) and retains the gated response.

\paragraph{Convolutional LSTM.}

The convolutional long short-term memory (ConvLSTM) baseline uses one recurrent layer with circular spatial padding. The same temporal refinement architecture as SLD produces 100 coarse frames, which are processed sequentially:
\begin{equation}
    (\mathbf{h}_k,\mathbf{c}_k)
    =
    \operatorname{ConvLSTM}
    (\widetilde{\mathbf{x}}_k,\mathbf{h}_{k-1},\mathbf{c}_{k-1}),
    \label{eq:convlstm_update}
\end{equation}
where \(\mathbf{h}_k\) and \(\mathbf{c}_k\) are the hidden state and memory state, respectively. At each output time, the hidden state and refined input frame are decoded through the wavelet and skip branches:
\begin{equation}
    \widehat{\mathbf{y}}^{\rm ConvLSTM,norm}_k
    =
    W_{\theta}^{\rm ConvLSTM}(\mathbf{h}_k)
    +
    W_{\rm skip}^{\rm ConvLSTM}(\widetilde{\mathbf{x}}_k).
    \label{eq:convlstm_decode}
\end{equation}

The summed output is transformed to physical units using the high-resolution training statistics.

Supplementary Notes~\ref{supp:model_details} and~\ref{supp:baseline_details} give the model configurations and parameter counts; Supplementary Note~\ref{supp:training_details} gives optimization and checkpoint settings.

\subsection{Evaluation metrics}
\label{subsec:metrics}

All metrics are computed on high-resolution predictions and targets in physical units. We report relative \(L_2\) error and mean absolute error (MAE) to summarize reconstruction accuracy. For a test set of trajectory segments \(\mathcal{D}_{\rm test}\), the relative \(L_2\) error is
\begin{equation}
    E_{L_2}
    =
    \frac{1}{|\mathcal{D}_{\rm test}|}
    \sum_{(\mathbf{x},\mathbf{y})\in\mathcal{D}_{\rm test}}
    \frac{
    \|\widehat{\mathbf{y}}-\mathbf{y}\|_2
    }{
    \|\mathbf{y}\|_2
    },
    \label{eq:rel_l2}
\end{equation}
and MAE is
\begin{equation}
    E_{\rm MAE}
    =
    \frac{1}{|\mathcal{D}_{\rm test}|}
    \sum_{(\mathbf{x},\mathbf{y})\in\mathcal{D}_{\rm test}}
    \frac{1}{N_{\rm dof}}
    \|\widehat{\mathbf{y}}-\mathbf{y}\|_1 ,
    \label{eq:mae}
\end{equation}
where \(N_{\rm dof}\) is the number of scalar values in one target segment. Both norms are computed over all output times, physical channels and high-resolution spatial grid points for each test segment before averaging over the test set.

We also evaluate the spatial distribution of Fourier power, the change in relative \(L_2\) error over time, and reconstruction error in regions grouped by their reference gradient magnitude. For a field \(q\), the two-dimensional Fourier power is
\begin{equation}
    P_q(\boldsymbol{\kappa})
    =
    |\mathcal{F}[q](\boldsymbol{\kappa})|^2 ,
\end{equation}
and the radial spectrum \(E_q(k)\) is obtained by averaging \(P_q(\boldsymbol{\kappa})\) over annular bins labelled by radial wavenumber \(k\). The spectral relative error is computed as
\begin{equation}
    E_{\rm spec}(k)
    =
    \frac{
    |E_{\widehat{q}}(k)-E_q(k)|
    }{
    E_q(k)+\varepsilon_{\rm spec}
    },
    \qquad
    \varepsilon_{\rm spec}=10^{-12}.
    \label{eq:spectral_error}
\end{equation}
Spectra are averaged over output times, physical channels and test segments unless otherwise stated.
% AUTHOR CHECK REQUIRED [REV-03]: State FFT normalization, annular-bin boundaries,
% treatment of the zero mode, and whether spectral errors are evaluated before
% or after averaging the spectra.

For RMSE grouped by gradient magnitude, we first compute the reference gradient magnitude
\begin{equation}
    g_{k,i,j}
    =
    \left[
    \sum_{c=1}^{C}
    \left(
    |D_x^{\rm HR} y_{k,c,i,j}|^2
    +
    |D_y^{\rm HR} y_{k,c,i,j}|^2
    \right)
    \right]^{1/2},
    \label{eq:gradient_magnitude}
\end{equation}

where \(D_x^{\rm HR}\) and \(D_y^{\rm HR}\) denote the periodic difference operators used only for metric evaluation on the high-resolution grid. Ground-truth space--time points are pooled over test segments and partitioned into deciles of \(g_{k,i,j}\). RMSE is then computed within each decile:
\begin{equation}
    {\rm RMSE}_m
    =
    \left[
    \frac{1}{|\Omega_m|C}
    \sum_{(k,i,j)\in\Omega_m}
    \sum_{c=1}^{C}
    \left(
    \widehat{y}_{k,c,i,j}
    -
    y_{k,c,i,j}
    \right)^2
    \right]^{1/2},
    \label{eq:gradient_rmse}
\end{equation}

where \(\Omega_m\) is the set of space--time points in gradient decile \(m\). Each decile contains approximately one tenth of the reference points, ordered from the smallest to the largest gradients. The resulting RMSE values show how reconstruction accuracy changes from smooth regions to steep transitions.

\section*{Acknowledgments}
This work was partially supported by the National Natural Science Foundation of China (72495131, 82441027), Guangdong Provincial Key Laboratory of Mathematical Foundations for Artificial Intelligence (2023B1212010001), Shenzhen Stability Science Program, and the Shenzhen Science and Technology Program under grant no. JCYJ20250604141043020.

\section*{Data availability}
The datasets generated in this study are publicly available on Hugging Face at \url{https://huggingface.co/datasets/ChenhaoSI/SLD-dataset}.

\section*{Code availability}
The code for dataset generation and for training and evaluating SLD is publicly available on GitHub at \url{https://github.com/S-Chenhao/SLD}.

\section*{Author contributions}
C.S. and M.Y. conceived the study. C.S. and Z.Q. developed the method, implemented the models, generated the datasets, performed the experiments and prepared the figures. C.S., Z.Q. and M.Y. analysed the results and wrote the manuscript. S.-T.Y. provided scientific guidance. R.W. and M.Y. supervised the project and revised the manuscript. C.S. and Z.Q. contributed equally to this work.

\section*{Competing interests}

The authors declare no competing interests.

\bibliography{sn-bibliography_ai4pde}% common bib file
\section{Supplementary}
% Supplementary numbering.
\setcounter{table}{0}
\setcounter{figure}{0}
\renewcommand{\thetable}{S\arabic{table}}
\renewcommand{\thefigure}{S\arabic{figure}}
\renewcommand{\tablename}{Supplementary Table}
\renewcommand{\figurename}{Supplementary Figure}
% ============================================================
\subsection{Notation}
\label{supp:notation}

Supplementary Table~\ref{tab:notation} summarizes the main symbols used in
the manuscript and Supplementary Information.

\begin{table}[t]
\centering
\small
\caption{
Main notation used in the manuscript and Supplementary Information.
}
\label{tab:notation}
\begin{tabular}{ll}
\toprule
Symbol & Description \\
\midrule
\(\mathbf{x}_{1:T_{\rm in}}\) & Low-resolution input trajectory segment \\
\(\mathbf{y}_{1:T_{\rm out}}\) & High-resolution target trajectory segment \\
\(\widehat{\mathbf{y}}_{1:T_{\rm out}}\) & Reconstructed high-resolution trajectory segment \\
\(T_{\rm in}\) & Number of low-resolution input frames, \(T_{\rm in}=25\) \\
\(T_{\rm out}\) & Number of high-resolution output frames, \(T_{\rm out}=100\) \\
\(r\) & Output-to-input sequence-length ratio, \(r=T_{\rm out}/T_{\rm in}=4\) \\
\(s\) & Spatial refinement factor, \(s=8\) \\
\(\tau\) & Low-resolution segment start index \\
\(C\) & Number of physical state variables, \(C=2\) in all systems \\
\(C_h\) & Number of latent/guidance channels in SLD, \(C_h=64\) \\
\(\widetilde{\mathbf{x}}_k\) & Refined coarse frame at output index \(k\) \\
\(\mathbf{U}_k\) & Guidance feature on the low-resolution grid at output index \(k\) \\
\(\mathbf{h}_k\) & Latent state after update \(k\) \\
\(G_\theta\) & Learned gated response in the SLD latent vector field \\
\(\mathcal{S}_{\psi}\) & Structured spatial operator acting on the latent state \\
\(D_x,D_y\) & Horizontal and vertical central differences with unit latent spacing \\
\(D_x^{\rm HR},D_y^{\rm HR}\) & Periodic high-resolution difference operators used for metric evaluation \\
\(\Delta\) & Five-point Laplacian with unit latent spacing \\
\(\boldsymbol{\alpha},\boldsymbol{\beta}_x,\boldsymbol{\beta}_y\) & Learnable channel-wise coefficients of \(\mathcal{S}_{\psi}\) \\
\(W_\theta\) & Haar wavelet synthesis branch for the latent state \\
\(W_{\rm skip}\) & Haar wavelet synthesis branch for the refined coarse input \\
\(S_1,\ldots,S_5\) & Fixed trajectory segments ordered by start index \\
\(\mathcal{L}_{\rm data}\) & Supervised high-resolution reconstruction loss \\
\(\mathcal{L}_{\rm phy}\) & PDE residual loss evaluated on reconstructed HR fields \\
UCLD & Unstructured Controlled Latent Dynamics baseline \\
\(B\) & Batch size \\
\(\boldsymbol{\mu}_{\rm LR},\boldsymbol{\sigma}_{\rm LR}\) & Channel-wise low-resolution mean and standard deviation of the training data \\
\(\boldsymbol{\mu}_{\rm HR},\boldsymbol{\sigma}_{\rm HR}\) & Channel-wise high-resolution mean and standard deviation of the training data \\
ConvLSTM & Convolutional long short-term memory baseline \\
\bottomrule
\end{tabular}
\end{table}

% ============================================================
\subsection{PDE systems, data generation and segment construction}
\label{supp:data_generation}

We generated datasets for three two-dimensional systems: Gray--Scott reaction--diffusion, the complex Ginzburg--Landau system and the vector Burgers system. Each dataset contains ten high-resolution trajectories generated from distinct initial conditions. We then downsampled each trajectory in time and space to construct the low-resolution inputs shared by all models.

All systems use periodic \(256\times256\) grids and two state variables, giving stored high-resolution arrays of shape \([2,T_{\rm HR},256,256]\). Fourfold temporal and eightfold spatial downsampling produce low-resolution arrays of shape \([2,T_{\rm LR},32,32]\). These arrays provide clean low-resolution inputs for all reported experiments.

\subsubsection{Gray--Scott reaction--diffusion system}
\label{supp:gs_generation}

The two-dimensional Gray--Scott system is written as
\begin{align}
    \frac{\partial A}{\partial t}
    &=
    D_A \Delta A
    -
    A B^2
    +
    f(1-A),\\
    \frac{\partial B}{\partial t}
    &=
    D_B \Delta B
    +
    A B^2
    -
    (k+f)B,
\end{align}
where \(A\) and \(B\) are the two chemical concentration fields. We use
\[
D_A=0.16,\qquad
D_B=0.08,\qquad
f=0.06,\qquad
k=0.062.
\]
The simulations are performed on a periodic \(256\times256\) grid with unit grid spacing \(\Delta x=1\) and time step \(\Delta t=1\). Time integration uses a fourth-order Runge--Kutta method.

The Laplacian is discretized by the fourth-order periodic finite-difference stencil
\begin{equation}
\begin{aligned}
\Delta_h q_{i,j}
=
\frac{1}{\Delta x^2}
\bigg[
&-5q_{i,j}
+
\frac{4}{3}
\left(
q_{i+1,j}+q_{i-1,j}+q_{i,j+1}+q_{i,j-1}
\right)
\\
&-
\frac{1}{12}
\left(
q_{i+2,j}+q_{i-2,j}+q_{i,j+2}+q_{i,j-2}
\right)
\bigg],
\end{aligned}
\label{eq:supp_laplacian_4th}
\end{equation}

with periodic wrapping in both spatial directions.

For each initial-condition seed \(m=1,\ldots,10\), the initial condition is constructed as
\begin{align}
A_0(i,j)
&=
(1-\rho)
+
\rho\,\xi_A(i,j),\\
B_0(i,j)
&=
\rho\,\xi_B(i,j),
\end{align}

Here, \(\xi_A\) and \(\xi_B\) are independent uniform random fields generated using seed \(m\), and \(\rho=0.1\). We then insert a centered square perturbation. Let \(N=256\), \(N_2=N/2\), and let \(r_{\rm patch}\) denote the integer half-width of the square. For \((i,j)\in[N_2-r_{\rm patch},N_2+r_{\rm patch})\times[N_2-r_{\rm patch},N_2+r_{\rm patch})\), we set
\[
A_0(i,j)=0.50,\qquad B_0(i,j)=0.25.
\]

Each Gray--Scott trajectory is simulated for 15,000 internal time steps. The initial frame is stored, and one frame is then stored every five internal steps, giving
\[
T_{\rm HR}=3001.
\]
The stored high-resolution trajectory has shape \([2,3001,256,256]\). Temporal downsampling keeps every fourth stored high-resolution frame, giving \(T_{\rm LR}=751\). Spatial downsampling is performed by applying three levels of anti-aliased pyramid downsampling with a binomial filter, giving \(256/2^3=32\) grid points per direction. The resulting low-resolution trajectory has shape \([2,751,32,32]\).

\subsubsection{Complex Ginzburg--Landau system}
\label{supp:gl_generation}

The two-dimensional complex Ginzburg--Landau system is represented by its real and imaginary components \(u\) and \(v\):
\begin{align}
    \frac{\partial u}{\partial t}
    &=
    \epsilon \Delta u
    +
    \nu u
    -
    (u^2+v^2)(\gamma_{\rm re}u-\gamma_{\rm im}v),\\
    \frac{\partial v}{\partial t}
    &=
    \epsilon \Delta v
    +
    \nu v
    -
    (u^2+v^2)(\gamma_{\rm re}v+\gamma_{\rm im}u).
\end{align}
We use
\[
\epsilon=0.004,\qquad
\nu=10.0,\qquad
\gamma_{\rm re}=10.0,\qquad
\gamma_{\rm im}=15.0.
\]
The computational domain is \([-1,1)\times[-1,1)\), discretized by a periodic \(256\times256\) grid. Thus \(\Delta x=2/256\). The same fourth-order periodic finite-difference Laplacian in Eq.~\eqref{eq:supp_laplacian_4th} is used, with the appropriate \(\Delta x\). Time integration uses a fourth-order Runge--Kutta method with internal time step
\[
\Delta t=2.0\times 10^{-4}.
\]

The deterministic part of the initial condition is
\begin{align}
u_0(x,y)
&=
10\,y\,\exp[-25(x^2+y^2)],\\
v_0(x,y)
&=
10\,x\,\exp[-25(x^2+y^2)].
\end{align}
For each initial-condition seed \(m=1,\ldots,10\), independent Gaussian perturbations are added to \(u_0\) and \(v_0\):
\begin{align}
u_0(x,y)
&\leftarrow
u_0(x,y)+\rho\,\eta_u(x,y),\\
v_0(x,y)
&\leftarrow
v_0(x,y)+\rho\,\eta_v(x,y),
\end{align}
where \(\eta_u\) and \(\eta_v\) are standard normal random fields generated using seed \(m\), and \(\rho=0.1\).

Each trajectory stores \(T_{\rm HR}=1200\) high-resolution frames. The initial frame is stored, and subsequent stored frames are separated by 84 internal Runge--Kutta steps, corresponding to an interval between stored frames of
\[
84\times 2.0\times10^{-4}=1.68\times10^{-2}.
\]
The stored high-resolution trajectory has shape \([2,1200,256,256]\). Temporal downsampling keeps every fourth stored frame, producing \(T_{\rm LR}=300\). Spatial downsampling uses the same three levels of anti-aliased pyramid downsampling as for Gray--Scott, reducing each spatial dimension from 256 to 32. The resulting low-resolution trajectory has shape \([2,300,32,32]\).

\subsubsection{Two-dimensional vector Burgers system}
\label{supp:burgers_generation}

The vector Burgers system is
\begin{align}
    \frac{\partial u}{\partial t}
    +
    u\frac{\partial u}{\partial x}
    +
    v\frac{\partial u}{\partial y}
    &=
    \nu \Delta u,\\
    \frac{\partial v}{\partial t}
    +
    u\frac{\partial v}{\partial x}
    +
    v\frac{\partial v}{\partial y}
    &=
    \nu \Delta v,
\end{align}
with viscosity
\[
\nu=0.005.
\]
The domain is \([-1,1)\times[-1,1)\), equivalently a periodic square of side length \(L=2\), discretized on a \(256\times256\) grid. The time step is
\[
\Delta t=10^{-3}.
\]
Spatial derivatives are computed spectrally using the discrete Fourier transform. Nonlinear advection terms are dealiased using the standard two-thirds truncation mask. The system is advanced in time using a fourth-order Runge--Kutta method.

For each initial-condition seed \(m=1,\ldots,10\), we first generate independent smooth Fourier fields. On the domain \([-L/2,L/2)^2\), define
\[
\phi_{ij}(x,y)
=
2\pi\left(\frac{i x}{L}+\frac{j y}{L}\right),
\]
with modes \(i,j\in\{-4,\ldots,4\}\), excluding \((i,j)=(0,0)\). The base initial fields are
\begin{align}
u_{\rm base}(x,y)
&=
\sum_{i=-4}^{4}\sum_{j=-4}^{4}
\left[
a^{(u)}_{ij}\sin\phi_{ij}(x,y)
+
b^{(u)}_{ij}\cos\phi_{ij}(x,y)
\right],\\
v_{\rm base}(x,y)
&=
\sum_{i=-4}^{4}\sum_{j=-4}^{4}
\left[
a^{(v)}_{ij}\sin\phi_{ij}(x,y)
+
b^{(v)}_{ij}\cos\phi_{ij}(x,y)
\right],
\end{align}
where the zero mode is omitted and the coefficients are independently sampled from a standard normal distribution. Each base field is rescaled to \([0.1,1.1]\) using the affine transformation implemented in the data generator. We then add smooth perturbations:
\begin{align}
u_0(x,y)
&=
u_{\rm base}(x,y)
+
\rho\,\zeta_u(x,y),\\
v_0(x,y)
&=
v_{\rm base}(x,y)
+
\rho\,\zeta_v(x,y),
\end{align}
where \(\rho=0.1\). The perturbation fields \(\zeta_u\) and \(\zeta_v\) are obtained by Gaussian filtering white noise with \(\sigma=4.0\), then subtracting the spatial mean and dividing by the spatial standard deviation. The final initial fields are clipped to \([-0.2,1.4]\). For trajectory \(m\), the seed for the base fields is \(2026+m\), and the seed for the perturbations is \(10000+m\).

Each Burgers trajectory stores \(T_{\rm HR}=1200\) high-resolution frames, including the initial frame and one frame after each Runge--Kutta step. The stored high-resolution trajectory has shape \([2,1200,256,256]\). Temporal downsampling keeps every fourth stored frame, giving \(T_{\rm LR}=300\). Spatial downsampling is performed by centered Fourier cropping from \(256\times256\) to \(32\times32\), followed by the amplitude correction
\[
q_{\rm LR}
\leftarrow
q_{\rm LR}
\frac{32^2}{256^2}.
\]
The resulting low-resolution trajectory has shape \([2,300,32,32]\). The summary of generated datasets is shown in Supplementary
Table~\ref{tab:dataset_summary}.

\begin{table}[t]
\centering
\small
\caption{
Summary of generated datasets. Each system contains ten trajectories from different initial-condition seeds. HR and LR denote high- and low-resolution trajectories.
}
\label{tab:dataset_summary}
\begin{tabular}{lcccc}
\toprule
System
& Spatial Domain
& HR shape
& LR shape
& Downsampling\\
\midrule
Gray--Scott
& \([0,256)^2\)
& \([2,3001,256,256]\)
& \([2,751,32,32]\)
&\(8\times\) spatial, \(4\times\) temporal \\
Ginzburg--Landau
& \([-1,1)^2\)
& \([2,1200,256,256]\)
& \([2,300,32,32]\)
&\(8\times\) spatial, \(4\times\) temporal \\
Burgers
& \([-1,1)^2\)
& \([2,1200,256,256]\)
& \([2,300,32,32]\)
& \(8\times\) spatial, \(4\times\) temporal \\
\bottomrule
\end{tabular}
\end{table}

\subsubsection{Trajectory segment construction}
\label{supp:segment_construction}

We train and evaluate all models on segments sampled from the stored trajectories. Each input contains 25 low-resolution frames with shape
$$
[25,2,32,32],
$$
and each target contains the corresponding 100 high-resolution frames with shape
$$
[100,2,256,256].
$$
Using zero-based stored indices, a low-resolution segment beginning at index \(\tau\) is paired with its target as
\[
\mathbf{x}_j=\mathbf{x}^{\rm LR}_{\tau+j},\quad j=0,\ldots,24,
\qquad
\mathbf{y}_k=\mathbf{y}^{\rm HR}_{4\tau+k},\quad k=0,\ldots,99.
\]
The input and target windows share their start time. The observed coarse frames correspond to target offsets \(0,4,\ldots,96\), and the last three targets are short extrapolations beyond the final input time. All 100 target frames are reconstructed from the complete input window. The physical start time is \(4\tau\,\delta t_{\rm HR}\), where the spacing between stored HR frames is \(\delta t_{\rm HR}=5\) for Gray--Scott, \(0.0168\) for Ginzburg--Landau and \(0.001\) for Burgers. The protocol for unseen initial conditions uses four windows beginning at low-resolution indices \(100,150,200\) and \(250\). The protocol for later times defines five ordered segments by
\[
\operatorname{start}_{\rm LR}(S_1,\ldots,S_5)
=
(100,150,200,225,250).
\]

Each \(S_i\) contains 25 low-resolution input frames and 100 high-resolution target frames. The fourfold factor describes both the ratio of stored sampling rates and the ratio of sequence lengths. The temporal module produces 100 refined frames on the low-resolution grid, as described in Supplementary Note~\ref{supp:temporal_refinement_details}. Within each protocol, SLD and all baselines use identical window definitions.

For unseen initial conditions, trajectories are partitioned by identity into seven for training, two for validation and one for testing in each fold.

For reconstruction at later times, all ten trajectories provide \(S_1\)--\(S_3\) high-resolution targets for parameter optimization. The later segments \(S_4\) and \(S_5\) supply validation and test targets. Within each repeated partition, five trajectories provide the validation windows and the other five provide the test windows. Every model receives the low-resolution inputs from the later windows being evaluated.

\begin{table}[t]
\centering
\small
\caption{
Segment start indices used for reconstruction at later times. LR and HR denote the start indices in the low- and high-resolution trajectories. Each segment contains 25 LR frames and 100 HR frames. The protocol for unseen initial conditions uses starts 100, 150, 200 and 250.
}
\label{tab:segment_times}
\begin{tabular}{lccc}
\toprule
Segment & LR start index & HR start index & Usage \\
\midrule
\(S_1\) & 100 & 400 & training/evaluation at early times \\
\(S_2\) & 150 & 600 & training/evaluation at early times \\
\(S_3\) & 200 & 800 & training/evaluation at early times \\
\(S_4\) & 225 & 900 & evaluation in the first later window \\
\(S_5\) & 250 & 1000 & evaluation in the second later window \\
\bottomrule
\end{tabular}
\end{table}

\begin{table}[t]
\centering
\small
\caption{Summary of data partitioning in the two evaluation protocols. Partitions are defined by trajectory identity; the reported aggregation unit is specified separately for each protocol.
}
\label{tab:splits}
\begin{tabular}{llll}
\toprule
Protocol & Training data & Validation data & Test data \\
\midrule
Unseen initial condition
& 7 trajectories
& 2 trajectories
& 1 held-out trajectory \\
Later-time reconstruction
& \(S_1\)--\(S_3\) from all 10 traj
& \(S_4\)--\(S_5\) from val traj
& \(S_4\)--\(S_5\) from test traj \\
\bottomrule
\end{tabular}
\end{table}

% ============================================================
\subsection{Model architecture details}
\label{supp:model_details}

SLD uses \(C_h=64\) latent channels, temporal factor \(r=4\) and spatial factor \(s=8\) in the main experiments. Supplementary Table~\ref{tab:sld_architecture} summarizes the tensor shapes and trainable parameter counts. The forward pass consists of the following steps:
\begin{enumerate}
\item Normalize the 25 input frames and resample them to 100 positions. A temporal convolutional network adds a learned correction, producing \(\widetilde{\mathbf{x}}\in\mathbb{R}^{100\times B\times2\times32\times32}\).
\item Encode each refined frame as \(\mathbf{U}_k\in\mathbb{R}^{B\times64\times32\times32}\), and initialize \(\mathbf{h}_0=P_\theta(\mathbf{U}_1)\).
\item For \(k=1,\ldots,100\), compute \(\mathbf{h}_k\) from \(\mathbf{h}_{k-1}\) and \(\mathbf{U}_k\) with an explicit midpoint update. Each evaluation of the vector field adds the learned gated response and the structured spatial term.
\item Decode \(\mathbf{h}_k\) and \(\widetilde{\mathbf{x}}_k\) with two separate Haar synthesis branches. Add their outputs and transform the result to physical units, giving \(\widehat{\mathbf{y}}\in\mathbb{R}^{100\times B\times2\times256\times256}\).
\end{enumerate}
The encoder, latent update and synthesis branches share their respective parameters across output indices. The two synthesis branches have separate learned parameters.

\begin{table}[t]
\centering
\small
\caption{
SLD architecture used in the main experiments. The same configuration is used
for Gray--Scott, Ginzburg--Landau and Burgers systems and for all main
evaluation protocols.
}
\label{tab:sld_architecture}
\begin{tabular}{ll}
\toprule
Component & Setting \\
\midrule
Input segment & \(25\) LR frames, \(2\times32\times32\) per frame \\
Output segment & \(100\) HR frames, \(2\times256\times256\) per frame \\
Temporal upsampling & \(4\times\), linear interpolation plus temporal residual refinement \\
Spatial upsampling & \(8\times\), three-level Haar wavelet synthesis \\
Latent/guidance width & \(C_h=64\) \\
Guidance encoder & circular \(3\times3\) convolution, \(2\to64\) channels \\
Initial latent projection & \(1\times1\) convolution, \(64\to64\) channels \\
Gate path & depthwise circular \(7\times7\), GroupNorm(1), \(1\times1\), sigmoid \\
Candidate path & depthwise circular \(7\times7\), GroupNorm(1), \(1\times1\) expansion\\
Structured latent response & depthwise \(D_x,D_y,\Delta\) with learnable channel-wise coefficients \\
Latent update & fixed-step explicit midpoint update \\
Reconstruction head & \(3\times3\) convolution + inverse Haar synthesis \\
Skip branch & same wavelet synthesis form applied to refined LR input \\
\bottomrule
\end{tabular}
\end{table}

\subsubsection{Input normalization}
\label{supp:input_normalization}
% AUTHOR CHECK REQUIRED [REV-09]: Check the archived inverse normalization:
% forward normalization divides by sigma+1e-12 but the displayed inverse uses
% sigma. Report the implemented convention consistently.

Inputs and targets are normalized separately for each physical channel. Here, \(\boldsymbol{\mu}_{\rm LR}\) and \(\boldsymbol{\sigma}_{\rm LR}\) are the channel means and standard deviations of the low-resolution training data, and \(\boldsymbol{\mu}_{\rm HR}\) and \(\boldsymbol{\sigma}_{\rm HR}\) are the corresponding quantities for the high-resolution training data. The normalized input is
\begin{equation}
    \mathbf{x}^{\rm norm}
    =
    \frac{
    \mathbf{x}
    -
    \boldsymbol{\mu}_{\rm LR}
    }{
    \boldsymbol{\sigma}_{\rm LR}
    +
    10^{-12}
    }.
\end{equation}

The high-resolution training targets are normalized as
\begin{equation}
    \mathbf{y}^{\rm norm}
    =
    \frac{
    \mathbf{y}
    -
    \boldsymbol{\mu}_{\rm HR}
    }{
    \boldsymbol{\sigma}_{\rm HR}
    +
    10^{-12}
    }.
\end{equation}
The final prediction is transformed back to physical scale by
\begin{equation}
    \widehat{\mathbf{y}}
    =
    \widehat{\mathbf{y}}^{\rm norm}
    \boldsymbol{\sigma}_{\rm HR}
    +
    \boldsymbol{\mu}_{\rm HR}.
\end{equation}

For each channel, the mean and standard deviation are calculated over samples, time and both spatial dimensions of the relevant training split. SLD and all baselines use these same constants to normalize the training, validation and test data within that split.

\subsubsection{Temporal refinement and guidance encoding}
\label{supp:temporal_refinement_details}
% AUTHOR CHECK REQUIRED [REV-10]: Verify the 24j/99 interpolation coordinates
% against the archived code. Input frames lie at HR offsets 0,4,...,96, so the
% interpolated coordinates correspond to 96j/99 while target offsets are j.

The normalized input segment has shape
\[
    [T_{\rm in},B,2,H,W]=[25,B,2,32,32],
\]
Here, \(B\) is the batch size. The symbols \(\mathbf{x}\) and \(\widetilde{\mathbf{x}}\) denote normalized low-resolution fields. We include the batch dimension when specifying complete tensor shapes and omit it in per-frame maps such as \(E_\theta\), \(P_\theta\) and the synthesis branches.

The temporal module linearly interpolates the input along the sequence dimension from 25 to 100 positions, producing
\[
    \overline{\mathbf{x}}_{1:T_{\rm out}}
    \in
\mathbb{R}^{100\times B\times 2\times 32\times 32}.
\]
For zero-based output index \(j\in\{0,\ldots,99\}\), interpolation samples the input sequence at coordinate \(24j/99\), with input frames indexed from 0 to 24. Thus the first and last interpolated positions equal the first and last input frames. These 100 positions form the feature sequence used to condition reconstruction at the output indices.

At each spatial location, a temporal network predicts a correction to the interpolated sequence. It projects the two physical channels to 32 feature channels, applies two temporal convolutions with 32 channels and kernel size 5, each followed by GELU, and projects back to two channels. Zero padding preserves the sequence length. The correction is added to the interpolated sequence to obtain \(\widetilde{\mathbf{x}}_{1:T_{\rm out}}\). The same network parameters are used at all spatial locations.

The refined low-resolution frames are encoded into guidance features by a single circularly padded \(3\times3\) convolution:
\begin{equation}
    \mathbf{U}_k
    =
    E_\theta(\widetilde{\mathbf{x}}_k),
    \qquad
    E_\theta:\mathbb{R}^{2\times32\times32}
    \rightarrow
    \mathbb{R}^{64\times32\times32}.
\end{equation}
The initial latent state is obtained from the first guidance feature using a \(1\times1\) projection,
\begin{equation}
    \mathbf{h}_0
    =
    P_\theta(\mathbf{U}_1),
    \qquad
    P_\theta:\mathbb{R}^{64\times32\times32}
    \rightarrow
    \mathbb{R}^{64\times32\times32}.
\end{equation}
The first explicit midpoint update takes \(\mathbf{h}_0\) and \(\mathbf{U}_1\) as inputs and produces \(\mathbf{h}_1\), which is decoded as the first output.

\subsubsection{Guided latent dynamics}
\label{supp:guided_latent_details}

The latent state has shape
\[
    \mathbf{h}\in\mathbb{R}^{B\times64\times32\times32}.
\]
Given a latent state \(\mathbf{h}\) and guidance \(\mathbf{U}\), the vector field returns a tensor of the same shape:
\begin{equation}
    F_\theta(\mathbf{h},\mathbf{U})
    =
    G_\theta(\mathbf{h},\mathbf{U})
    +
    \mathcal{S}_{\psi}(\mathbf{h}).
\end{equation}
The learned response \(G_\theta\) uses reset and update gates. The concatenation \([\mathbf{h},\mathbf{U}]\) has 128 channels. A depthwise circular \(7\times7\) convolution filters each channel separately; group normalization with one group and a \(1\times1\) convolution then produce the gate inputs. Applying a sigmoid gives
\begin{equation}
    [\mathbf{r},\mathbf{z}]
    =
    \sigma\!\left(
    \mathcal{G}_\theta([\mathbf{h},\mathbf{U}])
    \right),
\end{equation}
where each gate has the same shape as \(\mathbf{h}\) and entries between zero and one. Group normalization with one group normalizes over all channels and spatial positions within each sample.

To construct the candidate state, we multiply \(\mathbf h\) by the reset gate \(\mathbf r\) and concatenate the result with \(\mathbf U\), giving 128 channels. A depthwise circular \(7\times7\) convolution followed by group normalization with one group processes this tensor. A \(1\times1\) convolution expands it to 256 channels, followed by GELU and a \(1\times1\) projection to 64 channels:
\begin{equation}
    \widetilde{\mathbf{h}}
    =
    \mathcal{C}_\theta
    ([\mathbf{r}\odot\mathbf{h},\mathbf{U}]).
\end{equation}
The learned response is
\begin{equation}
    G_\theta(\mathbf{h},\mathbf{U})
    =
    (1-\mathbf{z})
    \odot
    (\widetilde{\mathbf{h}}-\mathbf{h}).
\end{equation}
Thus \(1-\mathbf{z}\) weights the difference between the candidate and current states at each channel and grid point. Both gates and the candidate are recomputed for each evaluation of \(F_\theta\) in the midpoint update below. Here, \(\odot\) denotes element-wise multiplication.

\subsubsection{Structured latent response}
\label{supp:structured_response_details}

The structured latent response is
\begin{equation}
    \mathcal{S}_{\psi}(\mathbf{h})
    =
    \boldsymbol{\alpha}\odot\Delta \mathbf{h}
    +
    \boldsymbol{\beta}_x\odot D_x \mathbf{h}
    +
    \boldsymbol{\beta}_y\odot D_y \mathbf{h},
\end{equation}

Here, \(D_x\), \(D_y\) and \(\Delta\) are fixed depthwise \(3\times3\) operators applied independently to each latent channel. The horizontal and vertical central-difference kernels are
\begin{equation}
K_x
=
\begin{bmatrix}
0 & 0 & 0\\
-1/2 & 0 & 1/2\\
0 & 0 & 0
\end{bmatrix},
\qquad
K_y
=
\begin{bmatrix}
0 & -1/2 & 0\\
0 & 0 & 0\\
0 & 1/2 & 0
\end{bmatrix},
\end{equation}
and the five-point Laplacian kernel is
\begin{equation}
K_\Delta
=
\begin{bmatrix}
0 & 1 & 0\\
1 & -4 & 1\\
0 & 1 & 0
\end{bmatrix}.
\end{equation}
These fixed kernels act on a latent grid with unit spacing. The learned coefficients \(\boldsymbol{\alpha}\), \(\boldsymbol{\beta}_x\) and \(\boldsymbol{\beta}_y\) set the response scale independently for each channel. Each of the three coefficient vectors has 64 entries, one per latent channel. All entries are initialized to zero and learned jointly with the other model parameters. The Laplacian term responds to local curvature, and the two central-difference terms respond to directional variation in the latent features.

\subsubsection{Latent update}
\label{supp:latent_update_details}
The latent state is advanced with step size \(\Delta t=1\) in network time, so one update corresponds to one output index. The physical sampling intervals are specified in Supplementary Note~\ref{supp:data_generation}. For \(k=1,\ldots,T_{\rm out}\), the following two-stage explicit midpoint update uses the same guidance \(\mathbf U_k\) in both evaluations of the vector field:
\begin{align}
    \mathbf{k}_1
    &=
    F_\theta(\mathbf{h}_{k-1},\mathbf{U}_k),\\
    \mathbf{h}_{k-\frac{1}{2}}
    &=
    \mathbf{h}_{k-1}
    +
    \frac{1}{2}\Delta t\,\mathbf{k}_1,\\
    \mathbf{k}_2
    &=
    F_\theta(\mathbf{h}_{k-\frac{1}{2}},\mathbf{U}_k),\\
    \mathbf{h}_k
    &=
    \mathbf{h}_{k-1}
    +
    \Delta t\,\mathbf{k}_2.
\end{align}

The resulting states are stacked in temporal order. For decoding, the time and batch dimensions are combined, so the same synthesis branches process all output frames in one batched operation. No temporal mixing is performed by the synthesis branches.

\subsubsection{Wavelet synthesis and skip reconstruction}
\label{supp:wavelet_details}

SLD uses two Haar wavelet synthesis branches. The latent branch takes \(\mathbf h_k\in\mathbb R^{64\times32\times32}\), and the skip branch takes the refined coarse frame \(\widetilde{\mathbf x}_k\in\mathbb R^{2\times32\times32}\). Each branch produces a two-channel \(256\times256\) field, and their outputs are added:
\begin{equation}
    \widehat{\mathbf{y}}_k^{\rm norm}
    =
    W_\theta(\mathbf{h}_k)
    +
    W_{\rm skip}(\widetilde{\mathbf{x}}_k).
\end{equation}

Each head uses a weight-normalized circular \(3\times3\) convolution to predict the coefficients for three levels of Haar synthesis. The \(LL\) band supplies the coarsest approximation, and the \(LH\), \(HL\) and \(HH\) bands supply detail at each level. Inverse Haar synthesis doubles each spatial dimension at every level, taking the \(32\times32\) grid through \(64\times64\) and \(128\times128\) to \(256\times256\). Each branch therefore produces a field of shape \(2\times256\times256\), and their sum is the normalized prediction.

Each branch predicts \(128=2\times8^2\) coefficient channels on the \(32\times32\) grid. The first eight channels are split into the \(LL,LH,HL,HH\) bands, with two channels per band, and synthesized into a \(2\times64\times64\) field. The next 24 channels are split into three groups of eight. A channel-to-space rearrangement (pixel shuffle) with factor 2 converts each group into a two-channel detail band on the \(64\times64\) grid. These bands and the field from the first stage produce a \(2\times128\times128\) field. The last 96 channels are split into three groups of 32 and rearranged with factor 4 into two-channel detail bands on the \(128\times128\) grid. The final synthesis stage produces the \(2\times256\times256\) output. At each stage, the detail bands are ordered as \(LH,HL,HH\), and the output of the previous stage supplies the \(LL\) band. The inverse Haar transform uses
\begin{align}
x_{00}&=(LL+LH+HL+HH)/2, & x_{01}&=(LL-LH+HL-HH)/2,\\
x_{10}&=(LL+LH-HL-HH)/2, & x_{11}&=(LL-LH-HL+HH)/2,
\end{align}
and interleaves \(x_{00},x_{01},x_{10},x_{11}\) on the doubled spatial grid. The synthesized field becomes the \(LL\) input at the next level. The factor \(1/2\) is the normalization for the orthonormal Haar transform.

% ============================================================
% ============================================================
\subsection{Baselines and implementation details}
\label{supp:baseline_details}

The comparisons across all three PDE systems use UCLD, which replaces the complete SLD latent vector field with a learned convolutional field, and ConvLSTM, which uses a discrete recurrent update. The additional Gray--Scott comparison includes neural operators, image super-resolution and physics-informed super-resolution methods, together with an ablation of the structured local response.

All baselines use the same input and output resolutions, refinement factors, trajectory splits, normalization and evaluation metrics as SLD. In the comparisons across the three PDE systems, UCLD and ConvLSTM also share the supervised reconstruction loss, PDE residual loss when enabled, optimizer family, learning rate schedule, batch size and checkpoint selection rule. These two baselines have parameter counts close to SLD; the additional Gray--Scott methods compare a broader range of architectures under the protocol for unseen initial conditions.

For the controlled comparisons, we selected the baseline widths to approximate the trainable parameter count of SLD. SLD, UCLD and ConvLSTM contain 171,874, 173,672 and 178,762 trainable parameters, respectively.

\subsubsection{Main baselines used across the three PDE systems}
\label{supp:main_baselines}

\paragraph{Unstructured Controlled Latent Dynamics (UCLD).}

UCLD uses the same temporal refinement, guidance encoding, latent initialization and reconstruction architecture as SLD, with a latent width of 36. It replaces the complete latent vector field \(G_\theta+\mathcal{S}_\psi\) with
\begin{equation}
    F_{\rm UCLD}(\mathbf{h}_k,\mathbf{U}_k)
    =
    \gamma\,\mathcal{N}_{\phi}
    \left([\mathbf{h}_k,\mathbf{U}_k]\right),
    \label{eq:supp_ucld_vector_field}
\end{equation}
where brackets denote channel concatenation. The network \(\mathcal{N}_{\phi}\) consists of circularly padded \(3\times3\) convolutional layers with GELU activations and intermediate width 45. The learned scalar \(\gamma\) is initialized to \(0.10\).

State updates and reconstruction follow Supplementary Notes~\ref{supp:latent_update_details} and~\ref{supp:wavelet_details}, with \(F_{\rm UCLD}\) replacing the SLD vector field and the network parameters trained for UCLD. The separate ablation removes only \(\mathcal{S}_\psi\), retaining the gated response \(G_\theta\) and the remaining SLD architecture.

\paragraph{ConvLSTM.}

ConvLSTM uses the same temporal refinement architecture as SLD to construct 100 frames on the low-resolution grid. One convolutional LSTM layer processes these frames sequentially, using 35 hidden channels, circular spatial padding, a \(3\times3\) input kernel and a \(5\times5\) hidden kernel. At each output index, the hidden state and refined input frame are decoded by separate Haar wavelet synthesis branches. Their outputs are added to obtain the normalized prediction, following Supplementary Note~\ref{supp:wavelet_details}.

\subsection{Additional baselines for the representative Gray--Scott study}
\label{supp:additional_baselines}
The additional comparison uses the two-dimensional Gray--Scott system with unseen initial conditions. It includes neural operators, a method for image super-resolution, a physics-informed method for spatiotemporal super-resolution and an ablation of the structured local response in SLD. All methods use the same trajectory splits, input and output resolutions, normalization and evaluation metrics.

\paragraph{Fourier Neural Operator (FNO).}
FNO~\cite{li2021fourier} is used as a spectral neural operator with global Fourier-domain mixing. It uses the same \(4\times\) temporal refinement module as SLD. The spatial super-resolution head lifts each refined low-resolution frame to \(n_{\rm feats}=64\) channels, applies three FNO blocks of width 32 with GELU activations, and retains 20 and 17 Fourier modes in the two spatial directions. Each FNO block consists of a spectral convolution plus a pointwise \(1\times1\) convolution.

\paragraph{Laplace Neural Operator (LNO).}
LNO~\cite{cao2024laplace} is used as a neural operator with nonlocal spectral operator layers. It uses the same \(4\times\) temporal refinement module as SLD. The spatial head first lifts each temporally refined low-resolution frame to \(n_{\rm feats}=48\) channels using a shallow convolutional feature extractor, and then projects the features to an LNO width of 32. It applies three LNO blocks, each combining a pole-residue parameterized Laplace spectral operator with a pointwise \(1\times1\) convolution, followed by GELU activation.

\paragraph{MambaNO.}
MambaNO~\cite{zheng2024MambaNO} is used as a neural operator based on a state-space model. It uses the same \(4\times\) temporal refinement module as SLD. The spatial head lifts each refined low-resolution frame to \(n_{\rm feats}=64\) channels and applies three Mamba blocks with channel width 64 and hidden width 64. Each block arranges the feature vectors at the \(HW\) grid points into a sequence, processes it with a Mamba module, and adds the output of a circular \(3\times3\) convolution applied separately to each channel.

\paragraph{One-Step Effective Diffusion Network (OSEDiff).}
The implementation used for the OSEDiff~\cite{wu2024onestep} comparison reconstructs high-resolution images in one step using a spatial head with a U-Net structure and feature modulation. It uses the same \(4\times\) temporal refinement module as SLD. Each temporally refined low-resolution frame is first lifted to \(n_{\rm feats}=48\) channels by three circularly padded \(3\times3\) convolutions with GELU activations. The spatial head then applies a compact encoder--bottleneck--decoder structure with base width 32.

\paragraph{Physics-informed deep super-resolution (PhySR).}
PhySR~\cite{Ren2023PhySR} is used as a physics-informed spatiotemporal super-resolution baseline that combines temporal interpolation, recurrent temporal correction and convolutional spatial reconstruction. In our implementation, the input sequence is first temporally upsampled by a factor of \(4\times\) using linear interpolation, followed by one ConvLSTM layer with \(n_{\rm feats}=32\) hidden channels and circularly padded \(3\times3\) convolutions. Spatial reconstruction is performed by two weight-normalized residual blocks with expansion ratio 4 and residual scaling factor 0.1.

\paragraph{SLD without the structured spatial term.}
This ablation removes \(\mathcal S_\psi(\mathbf h)\) from the latent vector field and retains the learned gated response \(G_\theta\). All other SLD modules, normalization settings and the training objective are unchanged. This comparison tests the contribution of the structured spatial term within SLD.

% ============================================================
\subsection{Training details and checkpoint selection}
\label{supp:training_details}

Within each PDE system and protocol, SLD, UCLD and ConvLSTM use the same optimization procedure. Their objective is the supervised high-resolution reconstruction loss, with an additional PDE residual loss in runs that use that term. Such runs begin with reconstruction loss alone and then gradually increase the residual weight.

\paragraph{Optimization.}

Models are trained for \(1{,}000\) epochs with a batch size of \(4\). We use Adam with an initial learning rate of \(5\times10^{-3}\), weight decay of \(10^{-5}\) and cosine annealing to a minimum learning rate of \(10^{-7}\) over the full training run. The gradient norm is clipped at \(0.5\). These settings are shared by all models within each controlled comparison.

For runs using the PDE residual loss, the first 300 epochs use the data term alone. From epoch 301, the residual weight increases linearly from zero to \(\lambda_{\rm phy}=5\times10^{-3}\) over 300 epochs and then remains fixed. The residual is evaluated on all 100 predicted frames. Predictions are transformed to physical units before the mean absolute reconstruction loss and the PDE residual loss are evaluated.

\paragraph{Validation and checkpointing.}

Validation loss and relative \(L_2\) error are recorded after each epoch. The checkpoint with the lowest validation relative \(L_2\) error is used for testing. Training, validation and test assignments follow Supplementary Note~\ref{supp:segment_construction} and Table~\ref{tab:splits}.

The five fixed partition definitions and random seeds used for the reported results are distributed with the evaluation code so that the validation--test assignments can be reproduced exactly.
% ============================================================
\subsection{Numerical results for the main protocols}
\label{supp:numerical_results}

Supplementary Tables~\ref{tab:taskA_numerical} and~\ref{tab:taskB_numerical} report reconstruction errors for unseen initial conditions and later times, respectively. Supplementary Table~\ref{tab:gs_broader_numerical} reports the additional Gray--Scott baselines and the ablation of the structured local response. For unseen initial conditions, errors are first averaged over the four test segments of each held-out trajectory. The mean and standard deviation (s.d.) then summarize the ten resulting values, with each trajectory serving once as the test identity.
For reconstruction at later times, each partition contains five test trajectories and two test segments per trajectory. Errors are averaged over these ten segments to obtain one value per partition. The reported mean and s.d. summarize the five partition values (\(n=5\) repeated partitions), including the variation caused by their validation and test assignments. Some trajectory identities occur in multiple partitions. Figure~\ref{fig:later_time}\textbf{c} displays the mean error curves and their s.d. across the five partitions for the concatenated \(S_4\)--\(S_5\) sequence.

The UCLD and ConvLSTM configurations are given in Supplementary Note~\ref{supp:baseline_details}. Lower errors indicate more accurate reconstruction.

\begin{table}[t]
\centering
\small
\caption{Numerical results for reconstruction from unseen initial conditions. Rel. \(L_2\) denotes relative \(L_2\) error and MAE denotes mean absolute error; lower values indicate better reconstruction. Errors are averaged over the four segments of each held-out trajectory. Values are mean \(\pm\) s.d. over the ten resulting trajectory means (\(n=10\)), with every trajectory serving once as the test identity.
}
\label{tab:taskA_numerical}
\begin{tabular}{llcc}
\toprule
PDE system & Method & Rel. \(L_2\) \(\downarrow\) & MAE \(\downarrow\) \\
\midrule
Gray--Scott
& SLD
& \(\mathbf{1.855\times10^{-2}\pm3.970\times10^{-3}}\)
& \(\mathbf{4.160\times10^{-3}\pm8.150\times10^{-4}}\) \\
& UCLD
& \(5.224\times10^{-2}\pm1.005\times10^{-2}\)
& \(1.570\times10^{-2}\pm3.284\times10^{-3}\) \\
& ConvLSTM
& \(4.550\times10^{-2}\pm5.852\times10^{-3}\)
& \(1.048\times10^{-2}\pm1.427\times10^{-3}\) \\
\midrule
Ginzburg--Landau
& SLD
& \(\mathbf{1.820\times10^{-2}\pm5.127\times10^{-3}}\)
& \(\mathbf{6.017\times10^{-3}\pm1.866\times10^{-3}}\) \\
& UCLD
& \(4.347\times10^{-2}\pm3.348\times10^{-3}\)
& \(1.390\times10^{-2}\pm7.690\times10^{-4}\) \\
& ConvLSTM
& \(4.561\times10^{-2}\pm2.439\times10^{-3}\)
& \(1.189\times10^{-2}\pm5.516\times10^{-4}\) \\
\midrule
Burgers
& SLD
& \(\mathbf{2.829\times10^{-3}\pm7.877\times10^{-4}}\)
& \(\mathbf{9.321\times10^{-4}\pm3.862\times10^{-4}}\) \\
& UCLD
& \(8.268\times10^{-3}\pm2.345\times10^{-3}\)
& \(1.736\times10^{-3}\pm4.610\times10^{-4}\) \\
& ConvLSTM
& \(7.515\times10^{-3}\pm3.018\times10^{-3}\)
& \(2.069\times10^{-3}\pm5.321\times10^{-4}\) \\
\bottomrule
\end{tabular}
\end{table}

\begin{table}[t]
\centering
\small
\caption{Numerical results for reconstruction at later times. Rel. \(L_2\) denotes relative \(L_2\) error and MAE denotes mean absolute error; lower values indicate better reconstruction. Within each partition, errors are averaged over the two later segments from each of five test trajectories. Values are mean \(\pm\) s.d. of the five partition means (\(n=5\) repeated partitions).
}
\label{tab:taskB_numerical}
\begin{tabular}{llcc}
\toprule
PDE system & Method & Rel. \(L_2\) \(\downarrow\) & MAE \(\downarrow\) \\
\midrule
Gray--Scott
& SLD
& \(\mathbf{1.574\times10^{-2}\pm8.054\times10^{-4}}\)
& \(\mathbf{3.593\times10^{-3}\pm2.011\times10^{-4}}\) \\
& UCLD
& \(8.556\times10^{-2}\pm1.386\times10^{-2}\)
& \(2.997\times10^{-2}\pm5.108\times10^{-3}\) \\
& ConvLSTM
& \(3.613\times10^{-2}\pm8.396\times10^{-4}\)
& \(7.554\times10^{-3}\pm2.409\times10^{-4}\) \\
\midrule
Ginzburg--Landau
& SLD
& \(\mathbf{1.250\times10^{-2}\pm4.904\times10^{-4}}\)
& \(\mathbf{3.765\times10^{-3}\pm1.002\times10^{-4}}\) \\
& UCLD
& \(3.703\times10^{-2}\pm4.419\times10^{-3}\)
& \(1.345\times10^{-2}\pm1.844\times10^{-3}\) \\
& ConvLSTM
& \(4.575\times10^{-2}\pm3.246\times10^{-4}\)
& \(1.122\times10^{-2}\pm1.171\times10^{-4}\) \\
\midrule
Burgers
& SLD
& \(\mathbf{1.071\times10^{-3}\pm6.138\times10^{-5}}\)
& \(\mathbf{2.575\times10^{-4}\pm2.812\times10^{-5}}\) \\
& UCLD
& \(7.363\times10^{-3}\pm6.759\times10^{-4}\)
& \(1.794\times10^{-3}\pm1.599\times10^{-4}\) \\
& ConvLSTM
& \(6.467\times10^{-3}\pm1.332\times10^{-3}\)
& \(2.188\times10^{-3}\pm5.814\times10^{-4}\) \\
\bottomrule
\end{tabular}
\end{table}

\begin{table}[t]
\centering
\small
\caption{Numerical results for the additional Gray--Scott baselines and the local response ablation under the protocol for unseen initial conditions. Rel. \ensuremath{L_2} denotes relative \ensuremath{L_2} error and MAE denotes mean absolute error; lower values indicate better reconstruction. Values are mean \(\pm\) s.d. across ten held-out trajectories after averaging the four test segments within each trajectory.
}
\label{tab:gs_broader_numerical}
\begin{tabular}{lcc}
\toprule
Method & Rel. \ensuremath{L_2} \(\downarrow\) & MAE \(\downarrow\) \\
\midrule
SLD & \ensuremath{\mathbf{1.855\times10^{-2}\pm3.970\times10^{-3}}} & \ensuremath{\mathbf{4.160\times10^{-3}\pm8.150\times10^{-4}}} \\
SLD without local response & \ensuremath{3.720\times10^{-2}\pm1.433\times10^{-3}} & \ensuremath{7.825\times10^{-3}\pm2.621\times10^{-4}} \\
FNO & \ensuremath{1.011\times10^{-1}\pm6.480\times10^{-3}} & \ensuremath{3.090\times10^{-2}\pm2.057\times10^{-3}} \\
LNO & \ensuremath{1.056\times10^{-1}\pm2.224\times10^{-2}} & \ensuremath{3.322\times10^{-2}\pm7.878\times10^{-3}} \\
MambaNO & \ensuremath{5.607\times10^{-2}\pm8.965\times10^{-3}} & \ensuremath{1.636\times10^{-2}\pm3.496\times10^{-3}} \\
OSEDiff & \ensuremath{6.487\times10^{-2}\pm3.849\times10^{-3}} & \ensuremath{1.970\times10^{-2}\pm1.237\times10^{-3}} \\
PhySR & \ensuremath{7.925\times10^{-2}\pm5.671\times10^{-3}} & \ensuremath{2.083\times10^{-2}\pm1.201\times10^{-3}} \\
\bottomrule
\end{tabular}
\end{table}

% ============================================================
\subsection{Additional reconstruction diagnostics and visual evidence}
\label{supp:additional_diagnostics}

This section shows additional held-out reconstructions and maps of spatial errors. The fields are expressed in physical units and use the same predictions and references as the numerical tables.
% AUTHOR CHECK REQUIRED: For Supplementary Figs. S1--S3, specify the error sign convention and the trajectory, segment, time and channel represented by each row or column group. Error colorbars contain positive and negative values.

The same trajectory, segment, output time and state variable are shown for every compared method.
% AUTHOR CHECK REQUIRED [AC-13]: State the single plotting-code selection rule and
% the selected trajectory, segment, output time and state variable.

\subsubsection{Additional reconstructions for unseen initial conditions}
\label{supp:additional_unseen_visuals}

Supplementary Figs.~\ref{fig:supp_additional_unseen_heatmaps_GS}--\ref{fig:supp_additional_unseen_heatmaps_Burgers} show reconstructions from additional held-out trajectories and temporal segments under the protocol for unseen initial conditions.

\begin{figure}[!htb]
\centering
   \begin{minipage}{0.8\textwidth}
     \centering
     \includegraphics[width=\linewidth]{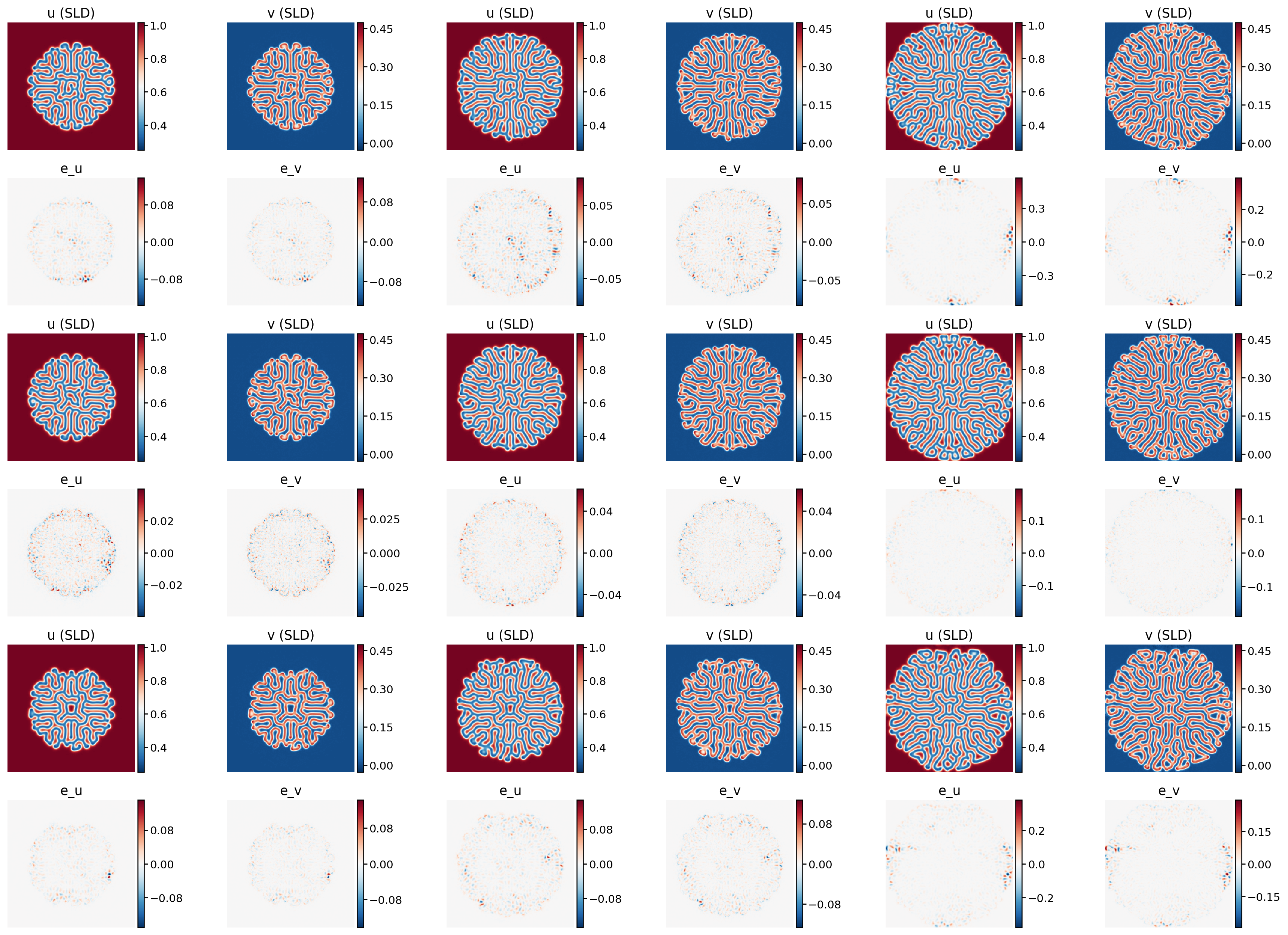}
   \end{minipage}
\caption{\textbf{Additional Gray--Scott reconstructions for unseen initial conditions.} SLD reconstructions of the two state variables and the corresponding error maps for selected held-out examples.
}
\label{fig:supp_additional_unseen_heatmaps_GS}
\end{figure}

\begin{figure}[!htb]
\centering
   \begin{minipage}{0.8\textwidth}
     \centering
     \includegraphics[width=\linewidth]{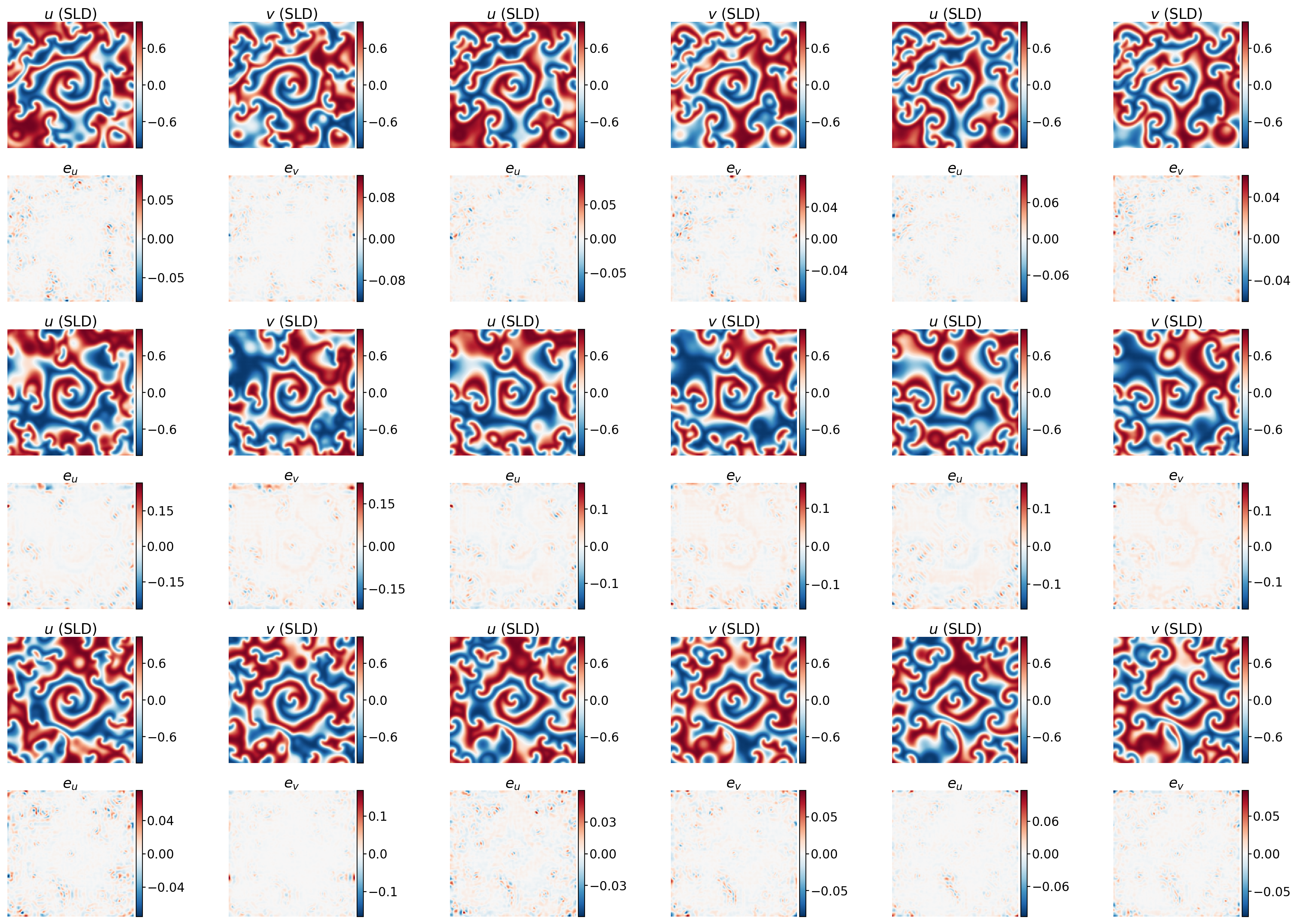}
   \end{minipage}
\caption{\textbf{Additional complex Ginzburg--Landau reconstructions for unseen initial conditions.} SLD reconstructions of the two state variables and the corresponding error maps for selected held-out examples.
}
\label{fig:supp_additional_unseen_heatmaps_GL}
\end{figure}

\begin{figure}[!htb]
\centering
   \begin{minipage}{0.8\textwidth}
     \centering
     \includegraphics[width=\linewidth]{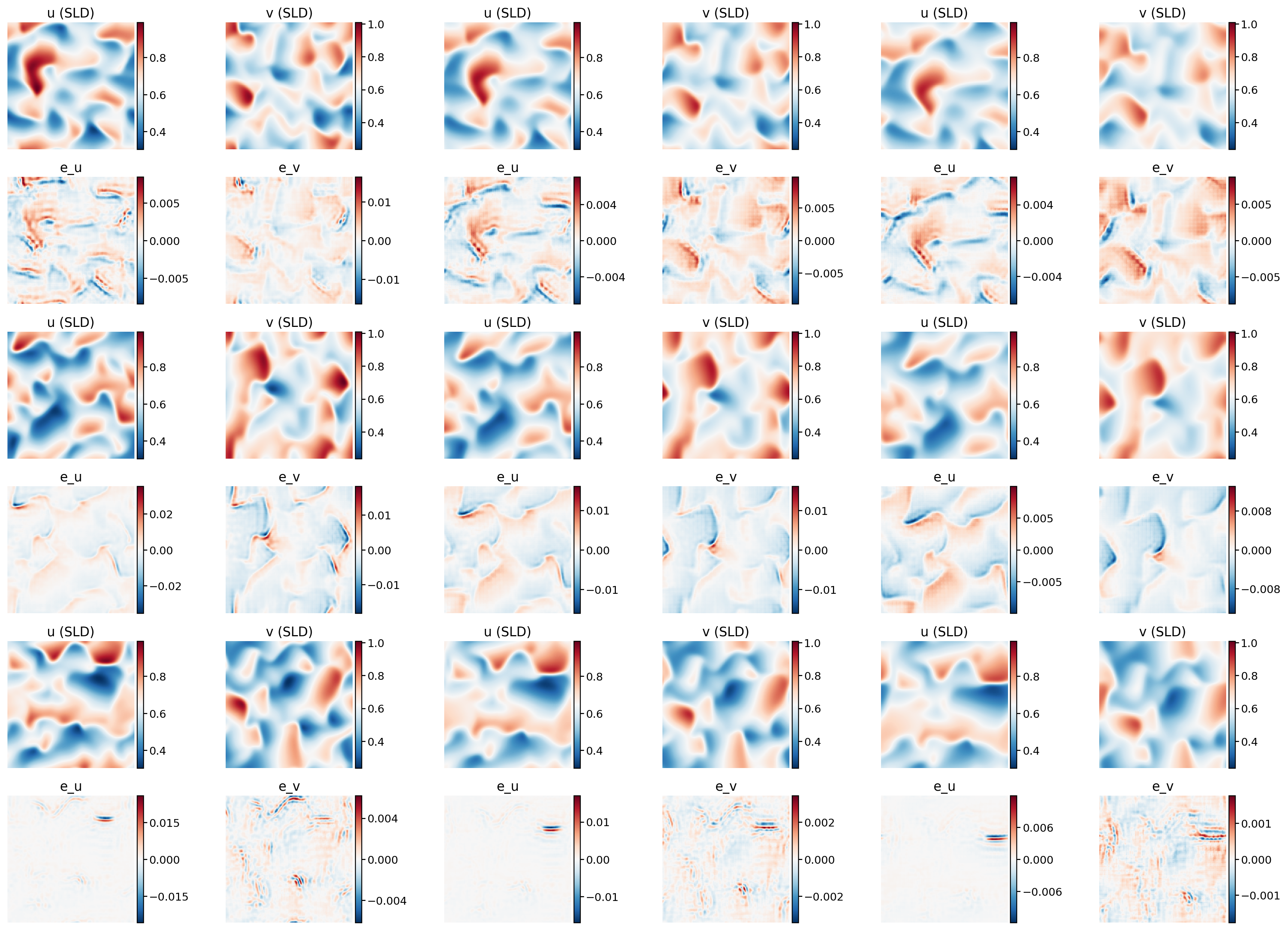}
   \end{minipage}
\caption{\textbf{Additional Burgers reconstructions for unseen initial conditions.} SLD reconstructions of the two state variables and the corresponding error maps for selected held-out examples.
}
\label{fig:supp_additional_unseen_heatmaps_Burgers}
\end{figure}

\subsubsection{Spatial maps of absolute error}
\label{supp:roi_error_maps}
Supplementary Fig.~\ref{fig:supp_roi_error_maps} compares the spatial distribution of absolute reconstruction errors for SLD, UCLD and ConvLSTM over the full domain. Each system and state variable uses a common color scale across methods.
\begin{figure}[!htb]
\centering
   \begin{minipage}{0.6\textwidth}
     \centering
     \includegraphics[width=\linewidth]{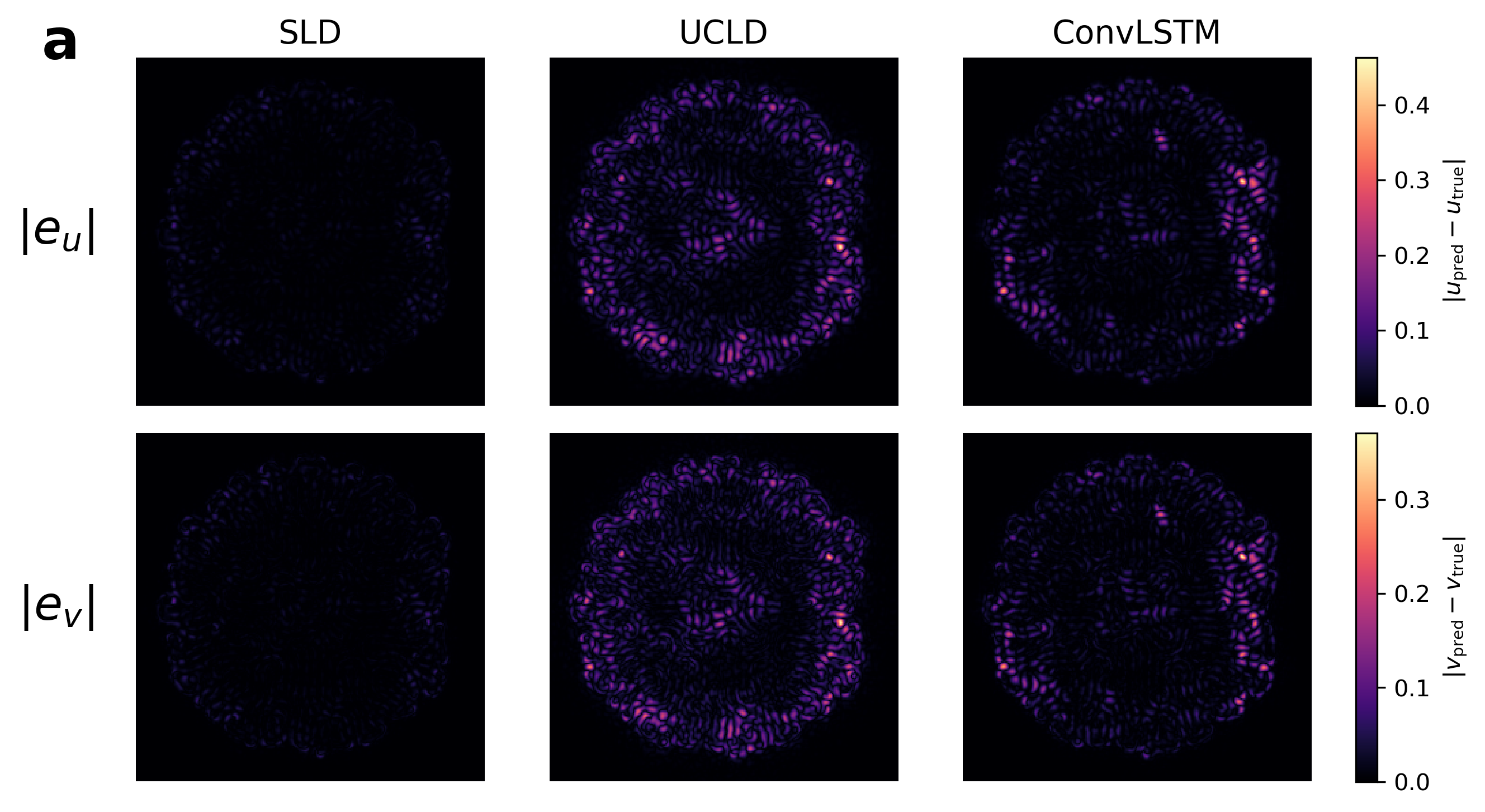}
   \end{minipage}

   \begin{minipage}{0.6\textwidth}
     \centering
     \includegraphics[width=\linewidth]{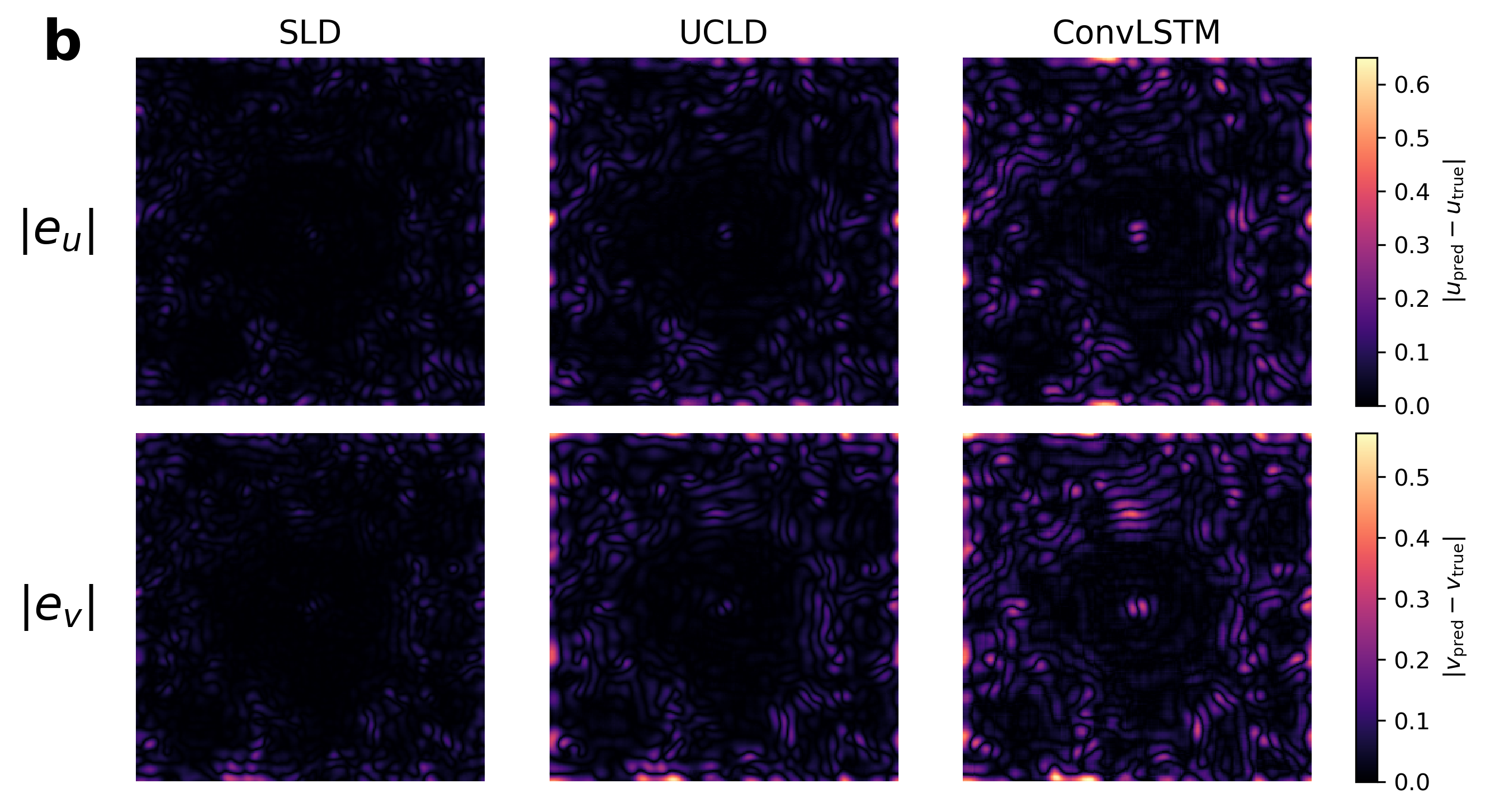}
   \end{minipage}

   \begin{minipage}{0.6\textwidth}
     \centering
     \includegraphics[width=\linewidth]{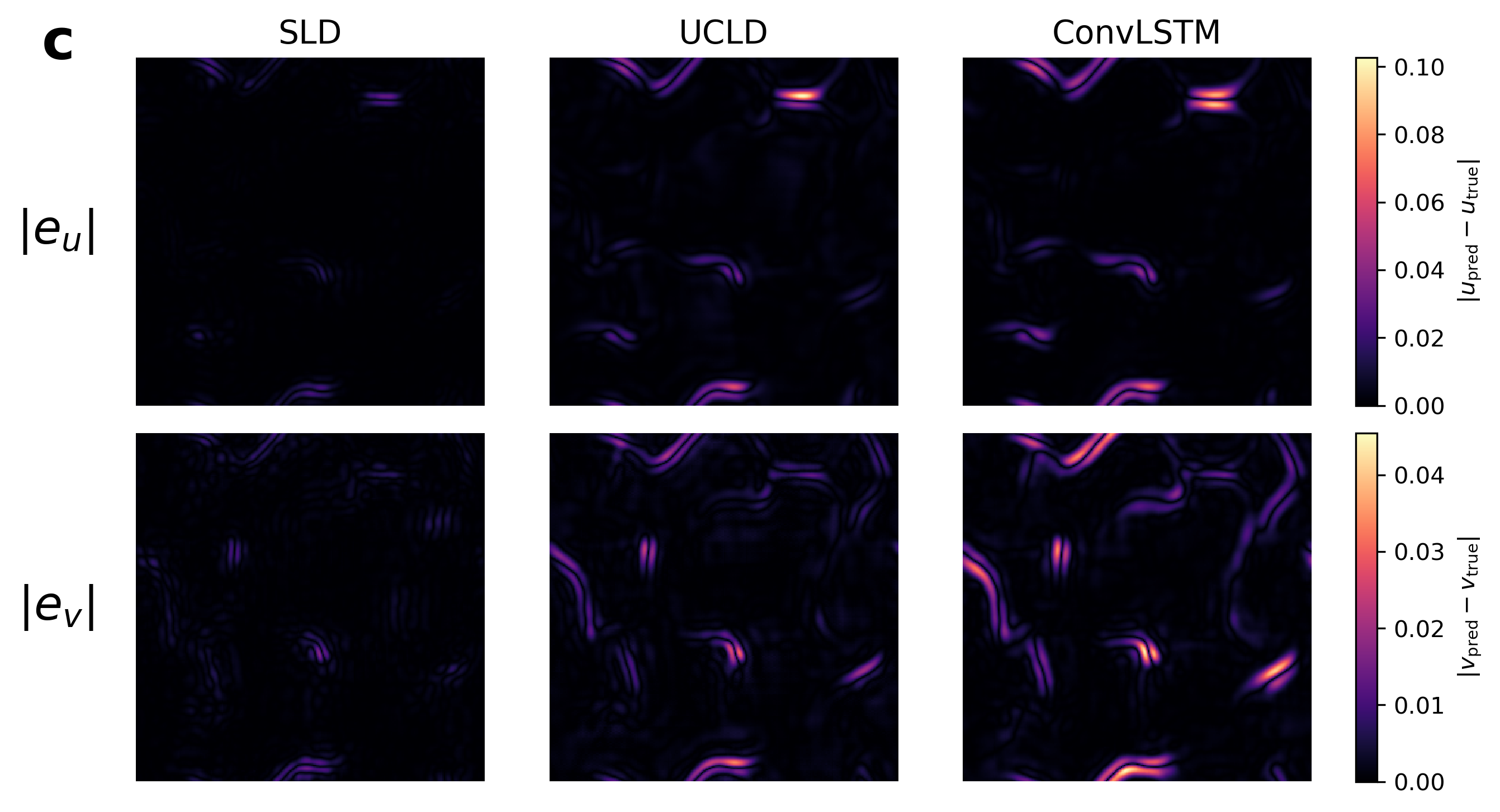}
   \end{minipage}

\caption{\textbf{Spatial maps of absolute reconstruction error.} \textbf{a,} Gray--Scott. \textbf{b,} Complex Ginzburg--Landau. \textbf{c,} Burgers. Columns compare SLD, UCLD and ConvLSTM, and rows correspond to the two state variables.
}
\label{fig:supp_roi_error_maps}
\end{figure}

% ============================================================
\subsection{Operator characterization of the structured local response}
\label{supp:operator_characterization}

We characterize the linear response \(\mathcal{S}_{\psi}\) on a periodic grid through its Fourier multiplier, operator norm bound and energy identity.

Let
\begin{equation}
\mathcal{S}_{\psi}\mathbf{h}
=
\boldsymbol{\alpha}\odot\Delta\mathbf{h}
+
\boldsymbol{\beta}_x\odot D_x\mathbf{h}
+
\boldsymbol{\beta}_y\odot D_y\mathbf{h},
\label{eq:supp_structured_response}
\end{equation}

where \(D_x\) and \(D_y\) are central differences, \(\Delta\) is the five-point Laplacian, and \(\boldsymbol{\alpha},\boldsymbol{\beta}_x,\boldsymbol{\beta}_y\) are real coefficients learned separately for each latent channel. The latent grid has unit spacing, as in the implementation.

\begin{theorem}[Fourier multiplier, norm bound and energy identity]
\label{thm:sld_response}
Consider
\(\mathbf{h}\in\mathbb{R}^{C_h\times N_y\times N_x}\) on a two-dimensional
periodic grid. For each latent channel \(c\), define
\begin{equation}
(D_x h_c)_{i,j}
=
\frac{h_{c,i,j+1}-h_{c,i,j-1}}{2},
\qquad
(D_y h_c)_{i,j}
=
\frac{h_{c,i+1,j}-h_{c,i-1,j}}{2},
\label{eq:supp_dx_dy}
\end{equation}
and
\begin{equation}
(\Delta h_c)_{i,j}
=
h_{c,i+1,j}
+
h_{c,i-1,j}
+
h_{c,i,j+1}
+
h_{c,i,j-1}
-
4h_{c,i,j}.
\label{eq:supp_lap}
\end{equation}
Let
\begin{equation}
\mathcal{S}_{\psi,c}
=
\alpha_c\Delta
+
\beta_{x,c}D_x
+
\beta_{y,c}D_y .
\label{eq:supp_channel_operator}
\end{equation}

Then \(\mathcal{S}_{\psi,c}\) is diagonalized by the discrete Fourier basis. For \(\xi=2\pi p/N_x\) and \(\eta=2\pi q/N_y\), where \(p=0,\ldots,N_x-1\) and \(q=0,\ldots,N_y-1\), its multiplier is
\begin{equation}
    \lambda_c(\xi,\eta)
    =
    -4\alpha_c
    \left[
    \sin^2\!\left(\frac{\xi}{2}\right)
    +
    \sin^2\!\left(\frac{\eta}{2}\right)
    \right]
    +
    \mathrm{i}\,\beta_{x,c}\sin\xi
    +
    \mathrm{i}\,\beta_{y,c}\sin\eta .
    \label{eq:sld_multiplier_supp}
\end{equation}
Moreover,
\begin{equation}
    \|\mathcal{S}_{\psi}\mathbf{h}\|_2
    \le
    \left(
    8\|\boldsymbol{\alpha}\|_\infty
    +
    \|\boldsymbol{\beta}_x\|_\infty
    +
    \|\boldsymbol{\beta}_y\|_\infty
    \right)
    \|\mathbf{h}\|_2 .
    \label{eq:sld_operator_bound_supp}
\end{equation}
Finally, for each channel \(c\),
\begin{equation}
    \left\langle
    h_c,
    \mathcal{S}_{\psi,c}h_c
    \right\rangle
    =
    -\alpha_c
    \left(
    \|\nabla_x^+ h_c\|_2^2
    +
    \|\nabla_y^+ h_c\|_2^2
    \right),
    \label{eq:sld_energy_identity_supp}
\end{equation}

where \((\nabla_x^+h_c)_{i,j}=h_{c,i,j+1}-h_{c,i,j}\) and \((\nabla_y^+h_c)_{i,j}=h_{c,i+1,j}-h_{c,i,j}\), with periodic indices. Inner products and norms use the sum over grid points. The central differences contribute zero to this inner product. The Laplacian contribution is nonpositive when \(\alpha_c\ge0\) and nonnegative when \(\alpha_c\le0\).
\end{theorem}

\begin{proof}
On a periodic grid, the discrete Fourier modes diagonalize all circulant
finite-difference operators. Let
\begin{equation}
\varphi_{\xi,\eta}(i,j)
=
\exp\{\mathrm{i}(\xi j+\eta i)\}.
\end{equation}
For the horizontal central difference,
\begin{align}
D_x\varphi_{\xi,\eta}
&=
\frac{
\exp\{\mathrm{i}[\xi(j+1)+\eta i]\}
-
\exp\{\mathrm{i}[\xi(j-1)+\eta i]\}
}{2}  \\
&=
\frac{e^{\mathrm{i}\xi}-e^{-\mathrm{i}\xi}}{2}
\varphi_{\xi,\eta}
=
\mathrm{i}\sin\xi\,
\varphi_{\xi,\eta}.
\end{align}
Similarly,
\begin{equation}
D_y\varphi_{\xi,\eta}
=
\mathrm{i}\sin\eta\,
\varphi_{\xi,\eta}.
\end{equation}
For the five-point Laplacian,
\begin{align}
\Delta\varphi_{\xi,\eta}
&=
\left(
e^{\mathrm{i}\eta}
+
e^{-\mathrm{i}\eta}
+
e^{\mathrm{i}\xi}
+
e^{-\mathrm{i}\xi}
-
4
\right)
\varphi_{\xi,\eta}.
\end{align}
Using
\begin{equation}
e^{\mathrm{i}a}+e^{-\mathrm{i}a}-2
=
-4\sin^2\!\left(\frac{a}{2}\right),
\end{equation}
we obtain
\begin{equation}
\Delta\varphi_{\xi,\eta}
=
-4
\left[
\sin^2\!\left(\frac{\xi}{2}\right)
+
\sin^2\!\left(\frac{\eta}{2}\right)
\right]
\varphi_{\xi,\eta}.
\end{equation}
Combining the multipliers of \(D_x\), \(D_y\) and \(\Delta\) gives
Eq.~\eqref{eq:sld_multiplier_supp}.

The operator norm bound follows from the maximum magnitudes of the individual
Fourier multipliers. For the five-point Laplacian,
\begin{equation}
\left|
-4
\left[
\sin^2\!\left(\frac{\xi}{2}\right)
+
\sin^2\!\left(\frac{\eta}{2}\right)
\right]
\right|
\le 8 .
\end{equation}
For the central differences,
\begin{equation}
|\mathrm{i}\sin\xi|\le 1,
\qquad
|\mathrm{i}\sin\eta|\le 1.
\end{equation}
Therefore, for each channel,
\begin{equation}
\|\mathcal{S}_{\psi,c}h_c\|_2
\le
\left(
8|\alpha_c|
+
|\beta_{x,c}|
+
|\beta_{y,c}|
\right)
\|h_c\|_2 .
\end{equation}

Taking the maximum coefficient magnitudes over channels, squaring the channel-wise inequalities and summing over channels gives Eq.~\eqref{eq:sld_operator_bound_supp}.

It remains to prove the energy identity. On a periodic grid, the central
difference operators are skew-adjoint:
\begin{equation}
D_x^\ast=-D_x,
\qquad
D_y^\ast=-D_y .
\end{equation}
For real-valued \(h_c\), this implies
\begin{equation}
\langle h_c,D_xh_c\rangle=0,
\qquad
\langle h_c,D_yh_c\rangle=0.
\end{equation}
The periodic five-point Laplacian satisfies the summation-by-parts identity
\begin{equation}
\langle h_c,\Delta h_c\rangle
=
-
\|\nabla_x^+ h_c\|_2^2
-
\|\nabla_y^+ h_c\|_2^2 .
\end{equation}
Substituting these identities into
\begin{align}
\left\langle h_c,\mathcal{S}_{\psi,c}h_c\right\rangle
&=
\alpha_c
\langle h_c,\Delta h_c\rangle
+
\beta_{x,c}
\langle h_c,D_xh_c\rangle
+
\beta_{y,c}
\langle h_c,D_yh_c\rangle
\end{align}
yields Eq.~\eqref{eq:sld_energy_identity_supp}.
\end{proof}

\paragraph{Interpretation.}

 For the isolated linear evolution \(\partial_t h_c=\mathcal{S}_{\psi,c}h_c\), the central-difference terms change the phase of each Fourier mode. The Laplacian term attenuates nonconstant modes when \(\alpha_c>0\) and amplifies them when \(\alpha_c<0\). Equation~\eqref{eq:sld_operator_bound_supp} bounds the norm of this linear operator, and Eq.~\eqref{eq:sld_energy_identity_supp} gives its contribution to the rate of change of latent-state energy. These properties characterize the structured spatial term; stability of the full update also depends on the learned gated response and the time discretization.
\end{document}